\documentclass[final,3p,times]{elsarticle}

\usepackage{multirow}
\usepackage{hyperref}
\usepackage{url}
\usepackage{booktabs}
\usepackage{subfigure}
\usepackage{graphicx}
\usepackage{algorithm}
\usepackage{algorithmic}
\usepackage{proof}
\usepackage{amsmath,amsfonts,amsthm,bm}
\usepackage{newfloat}
\usepackage{subcaption}
\usepackage{graphicx}
\usepackage{nicefrac}
\usepackage{listings}
\usepackage{wrapfig}
\usepackage{makecell}
\usepackage{bbm}
\usepackage{multirow}
\usepackage{xcolor}
\newcommand{\tcr}[1]{{\color{black}{#1}}}
\newcommand{\tcb}[1]{{\color{black}{#1}}}

\def\cS{{\mathcal{S}}}
\def\cD{{\mathcal{D}}}
\def\cA{{\mathcal{A}}}
\def\cN{{\mathcal{N}}}
\def\cM{{\mathcal{M}}}
\def\cO{{\mathcal{O}}}
\def\cT{{\mathcal{T}}}
\def\cL{{\mathcal{L}}}
\def\hT{{\widehat{\mathcal{T}}}}
\def\tT{{\widetilde{\mathcal{T}}}}
\def\VV{{\mathbb{V}}}
\def\EE{{\mathbb{E}}}
\def\RR{{\mathbb{R}}}
\def\hQ{{\hat{Q}}}

\def\hw{{\widehat{w}}}
\def\tw{{\widetilde{w}}}
\def\bI{{\mathrm{\mathbf{I}}}}
\def\hD{{\widehat{\mathcal{D}}}}

\newtheorem*{theorem*}{Theorem}
\newtheorem{theorem}{Theorem}
\newtheorem{lemma}{Lemma}
\newtheorem*{corollary*}{Corollary}
\newtheorem{corollary}{Corollary}
\newtheorem{definition}{Definition}

\journal{Artificial Intelligence}

\begin{document}

\begin{frontmatter}



\title{Pessimistic Value Iteration for Multi-Task Data Sharing in Offline \\Reinforcement Learning}

\author[shlab]{Chenjia Bai}\ead{baichenjia@pjlab.org.cn}
\author[northwestern]{Lingxiao Wang}
\author[tju]{Jianye Hao}
\author[yale]{Zhuoran Yang}
\author[shlab,nwpu]{Bin Zhao}
\author[nwpu]{Zhen Wang\corref{cor1}}\ead{w-zhen@nwpu.edu.cn}
\author[shlab,nwpu]{Xuelong Li\corref{cor1}}\ead{li@nwpu.edu.cn}

\affiliation[shlab]{organization={Shanghai Artificial Intelligence Laboratory},
    city={Shanghai},
    country={China}}
\affiliation[northwestern]{organization={Department of Industrial Engineering and Management Sciences, Northwestern University},
    city={Evanston},
	state={IL},
    country={USA}}
\affiliation[tju]{organization={Tianjin University},
    city={Tianjin},
    country={China}}
\affiliation[yale]{organization={Department of Statistics and Data Science, Yale University},
    city={New Haven},
	state={CT},
    country={USA}}
\affiliation[nwpu]{organization={School of Artificial Intelligence, OPtics and ElectroNics (iOPEN), Northwestern Polytechnical University},
    city={Xi'an},
    country={China}}
\cortext[cor1]{Corresponding Author}

\begin{abstract}
Offline Reinforcement Learning (RL) has shown promising results in learning a task-specific policy from a fixed dataset. However, successful offline RL often relies heavily on the coverage and quality of the given dataset. In scenarios where the dataset for a specific task is limited, a natural approach is to improve offline RL with datasets from other tasks, namely, to conduct Multi-Task Data Sharing (MTDS). Nevertheless, directly sharing datasets from other tasks exacerbates the distribution shift in offline RL. In this paper, we propose an uncertainty-based MTDS approach that shares the entire dataset without data selection. Given ensemble-based uncertainty quantification, we perform pessimistic value iteration on the shared offline dataset, which provides a unified framework for single- and multi-task offline RL. We further provide theoretical analysis, which shows that the optimality gap of our method is only related to the expected data coverage of the shared dataset, thus resolving the distribution shift issue in data sharing. Empirically, we release an MTDS benchmark and collect datasets from three challenging domains. The experimental results show our algorithm outperforms the previous state-of-the-art methods in challenging MTDS problems. See \url{https://github.com/Baichenjia/UTDS} for the datasets and code.
\end{abstract}





\end{frontmatter}


\section{Introduction}

Deep Reinforcement Learning (DRL) \citep{sutton-18} is a growing direction in Decision Making and plays an important role in broad Artificial Intelligence research. DRL has achieved remarkable success in a variety of tasks, including Atari games \cite{DQN-2015}, Go \cite{AlphaGo-2018}, and StarCraft \cite{AlphaStar-2019}. However, in most successful applications, DRL requires millions of interactions with the environment. In real-world applications such as navigation \citep{navigate-18} and healthcare \citep{healthcare}, acquiring a large number of samples by following a possibly suboptimal policy can be costly and dangerous. Alternatively, practitioners seek to develop \tcr{\emph{Offline RL} \citep{levine2020offline,lange2012batch} algorithms} that learn a policy based solely on an offline dataset, where the dataset is typically available. 
The most significant challenge in offline RL is the distribution shift issue, where the offline dataset does not match the current policy in optimization. Such a distribution shift leads to inaccurate policy evaluation, which hinders the performance of offline RL. To tackle this problem, most of the previous offline RL methods enforce policy constraints between the learned policy and the behavior policy that collects the \tcr{dataset \citep{bcq-2019,bear-2019}}. As a consequence, the performance of such methods relies heavily on the quality of the behavior policy. Alternatively, previous methods also measure the uncertainty of state-action pairs and utilize such uncertainty measurement to conduct conservative policy evaluation. Nevertheless, such methods rely on the coverage of the given dataset in the state-action space \citep{pevi-2021,xie2021policy}. As a result, if the offline dataset for a specific task is of low quality or has limited data coverage, the performance of offline RL algorithms is limited.

To this end, our research aims to develop a Multi-Task Data Sharing (MTDS) method, which enhances the offline RL on a specific task by utilizing datasets from other relevant tasks. The multi-task datasets are typically accessible in practical offline RL problems. For example, one can collect different datasets for various tasks with the same robot arm, where each task has a relatively small dataset. In such a scenario, directly training individual policies for each task with the corresponding dataset is insufficient for good performance since the data coverage for a single dataset is limited. Thus, a natural idea is to utilize offline datasets from other tasks in the same domain to help each task learn better (namely, MTDS). In particular, in learning a task $A_i\in A$ from domain $A$ with dataset $\cD_{A_i}$, an MTDS method constructs a mixed dataset $\hD_{A_i}$ by relabeling experience $e_j=(s,a,r_{A_j}(s,a),s')$ from task $A_j\in A$ to $e_{j\rightarrow i}=(s,a,r_{A_i}(s,a),s')$, where $i\neq j$. The shared tasks come from the same domain with the same dynamics; thus we only need to modify the rewards in the data sharing. We then train the offline policy for task $A_i$ on the mixed dataset $\hD_{A_i}$. 

Though such an MTDS process is easy to implement, naively sharing data can exacerbate the distribution shift between behavior and learned policies, which degrades performance compared to single-task training. In addition, when incorporating policy constraints, the different behavior policies of the mixed dataset drive the learned policy towards different directions, which results in conflicting gradient directions and instability in optimization. Previous methods propose Conservative Data Sharing (CDS) \citep{CDS-2021} to address this problem, which minimizes distribution shift by defining an effective behavior policy and only sharing data relevant to the main task. However, such a method hinges on policy constraints, while we insist policy coverage instead of the optimality of behavioral policy suits better in MTDS. In addition, the data selection process of CDS discards a large amount of shared data that could be potentially informative in training. 

\begin{figure}[t]
\centering
\includegraphics[width=1\columnwidth]{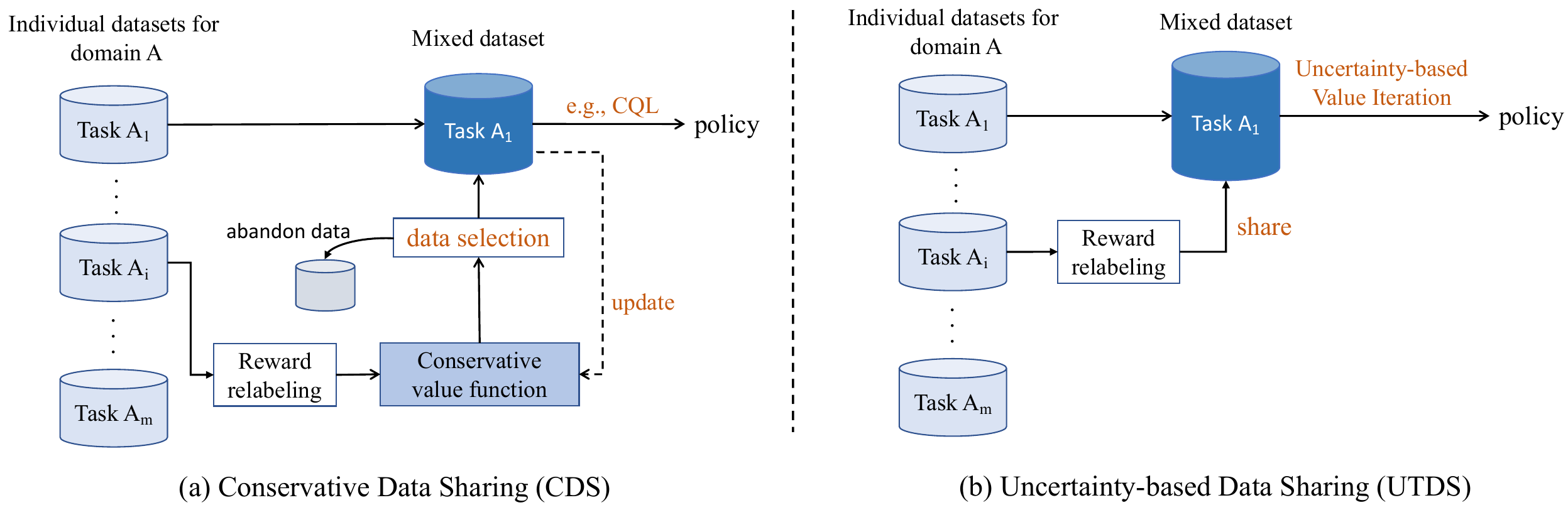}
\caption{The illustration of CDS and UTDS for MTDS in training task $A_1$. (a) CDS includes a data selection process through the learned conservative value function. The selected data is added to the mixed dataset.
(b) UTDS can share all data from other tasks without data selection. In policy training, UTDS performs pessimistic updates based on uncertainty in the large shared dataset.}
\label{fig:intro}
\end{figure}

In this paper, we propose \underline{U}ncer\underline{t}ainty-based \underline{D}ata \underline{S}haring (UTDS) algorithm, which allows arbitrary data sharing without data selection mechanisms. In particular, we first perform standard MTDS to obtain a mixed dataset with other tasks, then train an ensemble of $Q$-functions \citep{boot-2016} to provide uncertainty quantification for the mixed dataset. By measuring the uncertainty, we perform pessimistic value iteration to utilize the uncertainty as a penalty in offline training. Such an uncertainty quantifier considers the data coverage of the mixed dataset, which is less affected by the deviation between the behavior policies and the learned policy. In addition, we further penalize the out-of-distribution (OOD) actions within the support of the mixed dataset to improve the performance of the learned policy in the OOD region. A key factor for our proposed UTDS is that even if data sharing does not benefit the data coverage of the optimal policy, it does not degrade the UTDS learning results. Such a key factor makes UTDS inherently different from CDS \citep{CDS-2021}, which may have degraded performance without appropriate data selection. We illustrate the difference between CDS and UTDS in Figure~\ref{fig:intro}. 

A preliminary version of this work appeared as a spotlight conference paper \citep{PBRL-2022}, which mainly considers single-task offline RL. Nevertheless, we study the multi-task setting and data sharing problem in the present work. Our contributions are summarized as follows.
\begin{itemize}
\item The proposed UTDS algorithm provides a unified view for both single-task and multi-task offline RL learning without data selection. In particular, when the number of shared datasets is equal to one (namely $m=1$), UTDS degenerates to uncertainty-based single-task offline learning.
\item We provide a novel theoretical analysis in linear MDPs for MTDS. The analysis shows that the optimality gap of UTDS in data sharing is only related to the expected data coverage of the shared dataset on the optimal policy. Thus, UTDS is less affected by the change of the behavior policy in data sharing. 
\item Empirically, we conduct experiments on a suite of domains built on DeepMind Control Suite \citep{dmc-2018} and collect multi-task datasets to construct a benchmark for large-scale MTDS. We release the collected datasets and open-source code for the data-generation, and also reproduce several MTDS baselines for further utilization \footnote{See \url{https://github.com/Baichenjia/UTDS} for the datasets and code.}. 
\end{itemize}

\section{Preliminaries}

\subsection{Offline Reinforcement Learning} We consider an episodic MDP defined by the tuple $(\cS,\cA,T,r,P)$, where $\mathcal{S}$ is the state space, $\mathcal{A}$ is the action space, $T\in\mathbb{N}$ is the length of episodes, $r$ is the reward function, and $P$ is the transition distribution. The goal of RL is to find a policy $\pi$ that maximizes the expected cumulative rewards $\mathbb{E}\big[\sum_{i=0}^{T-1}{\gamma^t r_i}]$, where $\gamma\in[0,1)$ is the discount factor in episodic settings. The corresponding $Q$-function of the optimal policy satisfies the following Bellman operator, as
\begin{equation}
\mathcal{T}Q(s,a):=r(s,a)+\gamma \mathbb{E}_{s'\sim T(\cdot|s,a)}\big[\max_{a'} Q^-(s',a')\big],
\end{equation}
where $\theta$ is the parameters of $Q$-network. \tcb{We remark that relaxing the deterministic assumption of reward function to a stochastic one does not affect our method. Nevertheless, a deterministic reward $r(s,a)$ is quite standard in offline RL literature, and existing offline RL benchmarks also have deterministic reward functions.} In Deep RL (DRL), the $Q$-value is updated by minimizing the TD-error, namely $\mathbb{E}_{(s,a,r,s')}[(Q-\cT Q)^2]$. Empirically, the target $\mathcal{T}Q$ is typically calculated by a separate target-network parameterized by $\theta^-$ without gradient propagation \citep{DQN-2015}. In online RL, one typically samples the transitions $(s,a,r,s')$ through iteratively interacting with the environment. The $Q$-network is then trained by sampling from the collected transitions.

In offline RL, the agent is not allowed to interact with the environment. The experiences are sampled from an offline dataset. Naive off-policy methods such as $Q$-learning suffer from the \textit{distributional shift}, which is caused by different visitation distribution of the behavior policy and the learned policy. Specifically, the greedy action $a'$ chosen by the target $Q$-network in $s'$ can be an OOD-action since $(s',a')$ is scarcely covered by the dateset $\mathcal{D}_{\rm in}$. Thus, the value functions evaluated on such OOD actions typically suffer from significant extrapolation errors. Such errors can be further amplified through propagation and potentially diverges during training. 

\subsection{Multi-Task Data Sharing (MTDS)} A multi-task Markov Decision Process (MDP) is defined by the tuple $\cM=(\cS,\cA,P,\gamma,\{R_i,i\}_{i=1}^m)$, where $m$ is the number of tasks from the same domain, and we omit the domain notation $A$. We remark that all tasks share the same state space, action space, and transition function. Meanwhile, different tasks have different reward functions $\{R_i\}$ indexed by $i\in[m]$. Correspondingly, $Q^{\pi_i}_i(a|s)$ is the value function of task $i\in[m]$ with respect to the policy $\pi_i$. 

The multi-task dataset $\cD$ contains $m$ per-task dataset $\cD=\cup_{i=1}^m \cD_i$, where $\cD_i=\{(s,a,r_i,s')\}$ is the dataset for the $i$-th task and the reward function is returned by $R_i$. We denote the behavior policy for each dataset by $\{\pi_{\beta_1},\ldots,\pi_{\beta_m}\}$. For the $i$-th task, MTDS shares the dataset $\cD_j$ from the $j$-th task and relabels the rewards in $\cD_j$ with the reward function $R_i$. We denote the relabeled dataset from task $j\rightarrow i$ by $\cD_{j\rightarrow i}$. Then $\hD_i = \cD_i\cup \cD_{j\rightarrow i}$ denotes the shared dataset that contains all shared data from task $j$, where $j\neq i$. To conduct data sharing, we assume that the reward functions for all tasks are known to us, which is commonly achievable in robotic locomotion and manipulation tasks \citep{dmc-2018}. In special cases, if such reward functions are otherwise unknown, one can learn a parametric reward function based on the dataset. 

\subsection{Conservative Data Sharing (CDS)}
Previous work shows that naively sharing data to construct $\hD_i$ can degrade performance \citep{CDS-2021}. The reason is that data sharing increases the policy divergence $D(\pi_i\|\hat\pi_{\beta_i})(s)$ between the learned policy and the behavior policy as the shared data modifies $\hat\pi_{\beta_i}$. To this end, CDS optimizes the following objective to select the shared data, as
\begin{equation}\label{eq:cds-eq}
J_{\rm cds}=\max_{\hat\pi_{\beta_i}}\EE_{s\sim \hD_i}\big[\EE_{a\sim \pi_i}[Q(s,a)]-\alpha D(\pi_i\|\hat\pi_{\beta_i})(s)\big],
\end{equation}
where CDS selects data to obtain a behavior policy with small policy divergence. Theoretically, let $\pi^{\rm cds}_i$ and $\pi^*_{\beta_i}$ be policies that maximize Eq.~\eqref{eq:cds-eq}, then $w.p. \geq 1-\delta$, the policy improvement bound can be quantified as $J(\pi^{\rm cds}_i)\geq J(\pi_{\beta_i})-\delta_{\rm cds}
$, where
\begin{equation}\label{eq::CDS_bound111}
\delta_{\rm cds}\approx \cO\big(\nicefrac{1}{(1-\gamma)^2}\big) \EE_{s}\Big[\sqrt{\big(D(\pi^{\rm cds}_i\big\|\pi^*_{\beta_i})(s)+1\big)\:\big/\:\big|\hD_i(s)\big|} ~\Big] - \big[\alpha D(\pi^{\rm cds}_i,\pi^*_{\beta_i})+J(\pi^*_{\beta_i})-J(\pi_{\beta_i})\big].
\end{equation}
According to Eq.~\eqref{eq::CDS_bound111}, the learned policy should be close to the behavior policy to reduce this bound. We remark that such a lower bound is closely related to imitation learning rather than policy optimization since it only cares about the relationship to the behavior policy $\pi^*_{\beta_i}$. In contrast, our analysis of UTDS provides the optimality gap between the learned policy and the optimal policy, which provides useful guidance for data sharing without policy constraints.


In CDS implementation, optimizing $J_{\rm cds}$ can be approximately achieved by selecting transitions with a conservative $Q$-value more than the top $k$-th quantiles. Nevertheless, since the shared dataset changes with the update of $Q$-function, the learning process becomes unstable with the change of the effective behavior policy per update. Choosing a small $k$ (e.g., 10\%) may overcome such a problem, while it drops $90\%$ of the shared dataset that may contain helpful information for RL.

\section{Method}
\label{sec:method}

In this section, we introduce our proposed UTDS method. In UTDS, we collect all data from the shared dataset without data selection and measure the uncertainty through an ensemble of $Q$-networks. UTDS then performs the pessimistic value update to penalize the state-action pairs with large uncertainties in the shared dataset, which results in a pessimistic data-sharing policy.

\subsection{Uncertainty Quantifier}

Considering a shared dataset $\hD_i=\cD_i \cup \cD_{j\rightarrow i}$ that contains data from $i$-th task and relabeled data from the $j$-th task. We parameterize the $Q$-function for task $i$ as $Q_i(s,a;\theta_i)$, where $\theta_i$ denotes the parameters. The value function can be obtained via the following target,
\begin{equation}
\hT Q_i(s,a)=r_i(s,a)+\gamma \EE_{s'\sim \hD_i,a'\sim \pi_i} \big[Q_i^-(s',a')\big],
\end{equation}
where we denote $\hT Q_i$ as the target corresponding to the experience $(s,a,r,s')$  sampled from the shared dataset $\hD_i$, and we select $a'$ following the policy $\pi_i$. The target-network $Q_i^-$ is parameterized by $\theta_i^-$, which is updated by $\theta_i^-\leftarrow (1-\tau)\theta_i + \tau\theta_i$ without gradient propagation. 

To acquire the uncertainty quantification, we train $N$ independent $Q$-networks $\{Q_i^1,\dots,Q_i^N\}$ with different random initializations and independent target networks. We then measure the uncertainty of $(s,a)\sim\hD_i$ by the standard deviation of the ensemble-$Q$ predictions, namely,
\begin{equation}\label{eq:uncertainty-gamma}
\Gamma_i(s,a)=\sqrt{\VV[Q^j_i(s,a)]}, \quad j\in[1,\dots,N],
\end{equation}
where $\VV$ is the variance calculator. The deviation among deep ensemble-$Q$ networks trained with independent initializations tends to be large if the corresponding state-action pair scarcely occurs in the dataset. Such a deviation shrinks as the sample size increases with the data sharing. Deep ensemble-$Q$ networks trained with independent initializations will typically find different solutions, which corresponds to a variety of $Q$-functions centered in large basins of attraction \citep{bayesian-2020}. 

\begin{figure}[t]
\centering
\includegraphics[width=0.6\columnwidth]{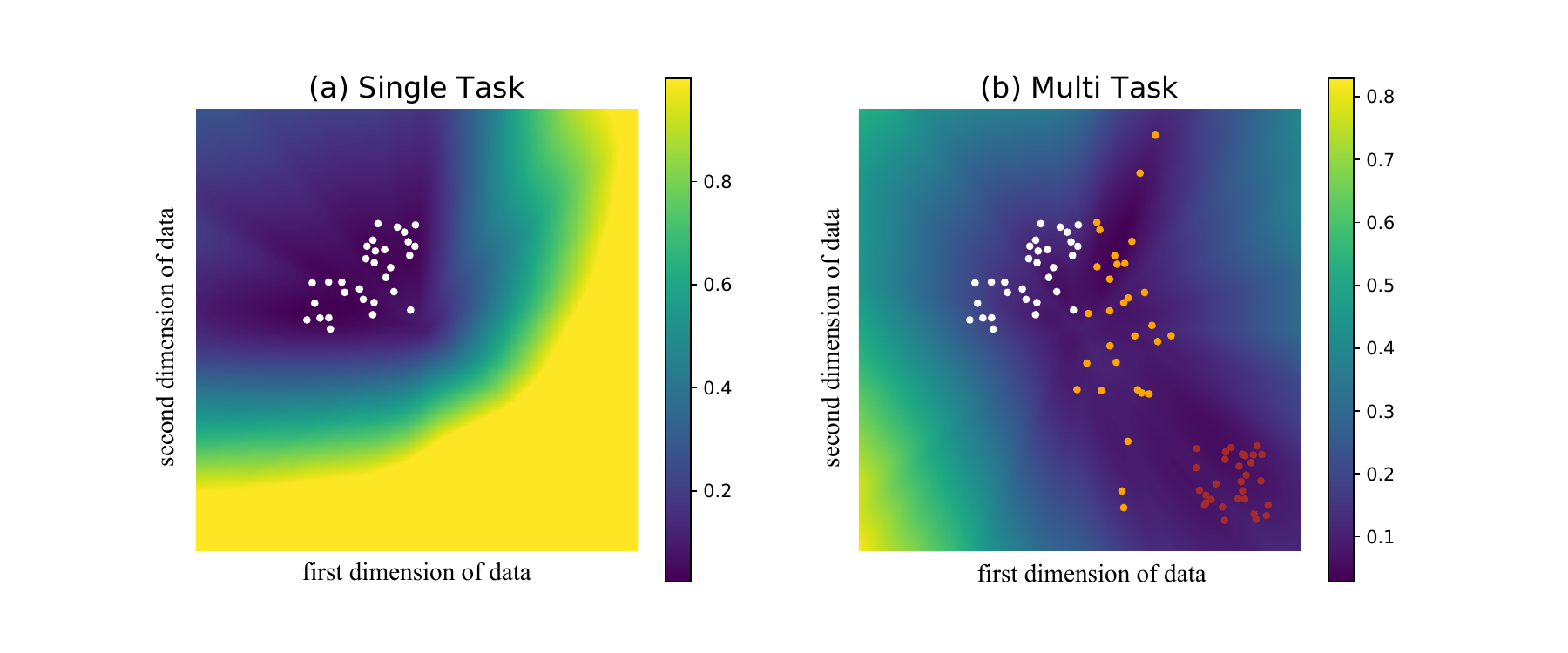}
\caption{An illustration of the uncertainty quantification of UTDS in (a) single-task dataset and (b) multi-task shared dataset. The multi-task datasets (i.e., the white, brown, and orange points) are generated by different Gaussian distributions. The uncertainty measured by the ensemble networks is represented by the color scales in the figures. The darker color means smaller uncertainty. As shown in the figures, sharing more data decreases the uncertainty.}
\label{fig:ucb-point}
\end{figure}

From the Bayesian perspective, the obtained ensemble-$Q$ networks is a non-parametric approximation of the posterior of the $Q$-function \citep{bayesian-2019,d2021repulsive}, which yields similar value on areas with rich data and diversely on areas with scarce data. Thus, the deviation among the ensemble $Q$-functions yields an epistemic uncertainty estimation, which we aim to utilize as a penalization in estimating the $Q$-functions. Here $\Gamma_t$ captures the radius of the confidence interval for the $Q$-function, which measures the uncertainty in the shared datasets. Recent works \citep{trust-2019,SGD-2022} also show that ensemble methods with stochastic gradient descent (SGD) and independent initialization are well generalized, thus capturing the epistemic uncertainty of the OOD data. We remark that although the ensemble-based method has been adopted in single-task offline RL \citep{EDAC-2021,PBRL-2022}, we present the first practical algorithm for uncertainty-based MTDS in offline RL.

We illustrate the uncertainty estimation with ensembles on a prediction task in $\mathbb{R}^2$ plane, as shown in Figure~\ref{fig:ucb-point}. The data are drawn from the Gaussian distributions, and the response is output by a random network. Figure~\ref{fig:ucb-point}(a) shows the uncertainty measured by the ensemble in the single-task dataset, and we find that the uncertainty is small for the state-action pairs near the data cluster. Nevertheless, since the single dataset has limited coverage, most areas still have large uncertainty. In contrast, in Figure~\ref{fig:ucb-point}(b), we share other two datasets and the uncertainty distribution changes. More areas are `known' and with low uncertainty, which helps reduce the uncertainty of the state-action pairs induced by the optimal policy. As we will discuss later, if the data sharing improves the data coverage on the trajectory induced by the optimal policy, UTDS enjoys a tighter optimality bound and better performance. In contrast, CDS that measures the policy divergence fails to adapt to the dataset similarity measured in reference to estimating the optimal value function, which means divergence quantification with trajectory distribution is inaccurate for uncertainty quantification.


\subsection{Pessimistic Value Iteration}

Based on the uncertainty quantifier $\Gamma_i$ in Eq.~\eqref{eq:uncertainty-gamma}, UTDS learns a pessimistic value function on the shared dataset. The Bellman target for each $Q$-function in the ensemble takes the following form,
\begin{equation}
\label{eq:bellman-UTDS}
\hT^{\rm UTDS} Q_i(s,a)=r_i(s,a)+\gamma \EE_{s'\sim \hD_i,a'\sim \pi_i} \big[Q_i^-(s',a')-\beta_1\Gamma_i(s',a')\big],
\end{equation}
where $\beta_1$ is a tuning parameter. Importantly, the uncertainty $\Gamma_i$ quantifies the deviation of a datapoint from the shared dataset, and the penalization enforces pessimism for the target-$Q$ value.

We remark that there are two options to penalize the $Q$-functions through the operator $\widehat\cT^{\rm in}$. In Eq.~(\ref{eq:bellman-UTDS}), we penalize the next-$Q$ value $Q^{-}_i(s', a')$ with the corresponding uncertainty $\Gamma_i(s',a')$. Alternatively, we can also penalize the immediate reward by $\hat{r}(s,a):=r(s,a)-\Gamma_i(s,a)$ and use $\hat{r}(s,a)$ in place of $r(s, a)$ for the target. Nevertheless, since the datapoint $(s, a) \in \cD_i$ lies on rich-data areas, the penalization  $\Gamma_i(s,a)$ on the immediate reward is usually very small thus having less effect in training. In UTDS, we penalize the uncertainty of $(s',a')$ in the next-$Q$ value.

Nevertheless, the target in Eq.~\eqref{eq:bellman-UTDS} only captures the uncertainty of experiences from the in-distribution data. Our empirical findings suggest that solely penalizing the uncertainty for in-distribution target is insufficient to control the OOD performance of the fitted $Q$-functions. In addition, we also need to penalize the value function for OOD actions. In UTDS, we sample OOD actions $a^{\rm ood}\sim \pi_i(s)$ following the learned policy, and the OOD states are still sampled from the shared dataset as $s\sim \hD_i$. We highlight that such OOD sampling requires only the offline dataset $\hD_i$ and does not require additional generative models or access to the simulator.

It remains to design the target for OOD samples. Since the transition $P(\cdot | s, a^{\rm ood})$ and reward $r(s, a^{\rm ood})$ are unknown, the true target of OOD sample is inapplicable. In UTDS, we propose a pseudo-target for the OOD datapoints. We quantify the uncertainty $\Gamma_i(s,a^{\rm ood})$ of the OOD data through the ensemble networks and use the following loss function to enforce pessimism for the OOD data, 
\begin{equation}
\label{eq:ood_loss}
\cL_{\rm ood}(s)=\EE_{s \sim \hD_i, a^{\rm ood} \sim\pi_i} \big[\big({\hT}^{\rm ood}Q_i(s,a^{\rm ood})-Q_i(s,a^{\rm ood})\big)^2\big],
\end{equation}
where ${\hT}^{\rm ood}Q_i(s,a^{\rm ood})$ is the pseudo-target for the OOD datapoints as,
\begin{equation}\label{eq:ood-target}
\hT^{\rm ood}Q_i(s,a^{\rm ood}):=Q_i(s,a^{\rm ood})-\beta_2\Gamma_i(s,a^{\rm ood}),
\end{equation}
where $\beta_2$ is a tuning parameter, and we remark ${\hT}^{\rm ood}Q_i(s,a^{\rm ood})$ is the target without gradient propagation. Our approach to OOD uncertainty quantification is inspired by PBRL \citep{PBRL-2022}, which we extend to the MTDS settings. As illustrated in Figure~\ref{fig:ucb-point}, $\Gamma_i$ provides well-calibrated uncertainty for both the in-distribution data and the OOD data.

We find that $\beta_2$ is important to the empirical performance. Specifically,
\begin{itemize}
\item At the beginning of training, both the $Q$-functions and the corresponding uncertainty quantifications are inaccurate. We use a large $\beta_2$ to enforce a strong regularization on OOD actions.
\item We then decrease $\beta_2$ exponentially with a factor $\alpha$ since the value estimation and uncertainty quantification becomes more accurate in training, which also prevents overly pessimism in the training process. We remark that a smaller $\beta_2$ requires more accurate uncertainty estimate for the pessimistic target estimation $\hT^{\rm ood} Q_i$. 
\end{itemize}

We remark that applying the operator for infinite times to $Q_i(s,a^{\rm ood})$ leads to $(\hT^{\rm ood})^\infty Q_i(s,a^{\rm ood})$, which equals
\begin{equation}\nonumber
Q_i(s,a^{\rm ood})- \sum\nolimits_{k=0}^{\infty} \alpha^k \beta_2\:\Gamma_i(s,a^{\rm ood})=Q_i(s,a^{\rm ood})-\frac{\beta_2}{(1-\alpha)} \: \Gamma_i(s,a^{\rm ood}).
\end{equation}
Thus, iteratively applying the OOD operator does not lead to the negative infinity target. In addition, a decaying parameter $\beta_2$ stabilizes the convergence of $Q_i(s,a^{\rm ood})$ in the training from the empirical perspective. Meanwhile, the operator $\hT^{\rm UTDS} Q_i(s,a)$ is a contraction mapping (see Appendix A.5 for a proof), which ensures that $Q_i(s,a)$ also converges to a fixed point. 

To conclude, when training the $i$-th task, the loss function for the $n$-th $Q$-function in the ensemble is
\begin{equation}\label{eq:UTDS-Q}
\cL_{\rm UTDS}^n=\EE_{s,a,r,s'\sim \hD_i} \Big[\Big(\hT^{\rm UTDS} Q_{\theta_i}^n(s,a) - Q_{\theta_i}^n(s,a)\Big)^2\Big] + \cL_{\rm ood}(s),
\end{equation}
where $\hT^{\rm UTDS}$ and $\cL_{\rm ood}$ are defined in 
Eq.~\eqref{eq:bellman-UTDS} and Eq.~\eqref{eq:ood_loss}, respectively. Each function $Q_{\theta_i}^n$ has its own target-network, where $n\in[N]$. The policy $\pi_{\phi_i}(s)$ for the $i$-th task is updated by minimizing the loss function as 
\begin{equation}\label{eq:UTDS-policy}
\cL_{\rm policy}=-\EE_{s\sim \hD_i,a\sim \pi_{\phi_i}} \Big[\min_{n=1,\ldots,N}Q_{\theta_i}^n(s,a)\Big],
\end{equation}
which takes a greedy action based on the minimum of the ensemble $Q$-functions. See Algorithm~\ref{alg:UTDS} for a summary of UTDS. 

\begin{algorithm}[t]
\caption{UTDS for training the $i$ -th task.}
\label{alg:UTDS}
\begin{algorithmic}[1]
\STATE {{\bf Initialize:} $N$ ensemble-$Q$ and target-$Q$ nets with parameter $\theta$ and $\theta^-$, initialize the policy $\pi_{\phi_i}$;}
\WHILE {\emph{not coverage}}
\STATE {Sample a batch of transition $(s,a,r,s')$ from $\hD_i$ that contains shared data from other tasks;}
\STATE {Perform OOD sampling based on the current policy to obtain $(s,a^{\rm ood})$ pairs;}
\STATE {Calculate the uncertainty $\Gamma_i(s,a)$ and $\Gamma_i(s,a^{\rm ood})$ based on the ensemble-$Q$ networks;}
\STATE {Preform pessimistic update via $\hT^{\rm UTDS}$ and $\hT^{\rm ood}$ to train the critic parameters $\theta_i$ via Eq.~\eqref{eq:UTDS-Q};}
\STATE {Train the pessimistic policy to train the actor parameters $\phi_i$ via Eq.~\eqref{eq:UTDS-policy};}
\ENDWHILE
\end{algorithmic}
\end{algorithm}

\section{Theoretical Analysis}

In this section, we consider linear MDPs \citep{lsvi-2020,pevi-2021} as a simplification, where we have access to a feature map of the state-action pair,  $\phi(s,a):\cS\times\cA\rightarrow\RR^d$. Meanwhile, the transition kernel and the reward function are linear in $\phi(s,a)$. We estimate the action value function by $Q_i(s,a)\approx \hw_i^{\top}\phi(s,a)$. See \S\ref{sec:theore} for the details. 

\subsection{UTDS in Linear MDPs}

From the theoretical perspective, an appropriate uncertainty quantification is essential to the provable efficiency in offline RL. Specifically, the $\xi$-uncertainty quantifier plays a central role in the analysis of both online and offline RL \citep{jaksch2010near,azar2017minimax, wang2020reward, lsvi-2020, pevi-2021, xie2021bellman, xie2021policy}.

\begin{definition}[$\xi$-Uncertainty Quantifier \citep{pevi-2021}]\label{def1}
The set of penalization $\{\Gamma_t\}_{t\in[T]}$ forms a $\xi$-Uncertainty Quantifier if it holds with probability at least $1 - \xi$ that
\begin{equation*}
|\widehat \cT V_{t+1}(s, a) - \cT V_{t+1}(s, a)| \leq \Gamma_t(s, a)
\end{equation*}
for all $(s, a)\in\cS\times\cA$, where $\cT$ is the Bellman equation and $\widehat \cT$ is the empirical Bellman equation that estimates $\cT$ based on the offline data.
\end{definition}

In linear MDPs \citep{lsvi-2020, wang2020reward, pevi-2021} where the transition kernel and reward function are assumed to be linear to the state-action representation $\phi:\mathcal{S}\times\mathcal{A}\rightarrow\mathbb{R}^d$, The following lower confidence bound (LCB)-penalty \citep{bandit-2011, lsvi-2020} is known to be a $\xi$-uncertainty quantifier for appropriately selected $\{\beta_t\}_{t\in[T]}$,
\begin{equation}
\label{eq:lcb}
\Gamma^{\rm lcb}(s_t,a_t)=\beta_t\cdot\big[\phi(s_t,a_t)^\top\Lambda_i^{-1}\phi(s_t,a_t)\big]^{\nicefrac{1}{2}},
\end{equation} 
where $\Lambda_i = \sum^{|\cD_i|}_{k=1}\phi(s^k_t, a^k_t)\phi(s^k_t, a^k_t)^\top + \lambda \cdot \bI$ is the covariance matrix corresponding to the dataset of task $i$, which plays the role of a pseudo-count intuitively. We remark that under such linear MDP assumptions, the penalty proposed in UTDS and $\Gamma^{\rm lcb}(s_t,a_t)$ in linear MDPs is equivalent under a Bayesian perspective. Specifically, we make the following claim.

\newtheorem{claim}{Claim}
\begin{claim}
\label{lem:lcb-main}
In linear MDPs, the proposed bootstrapped uncertainty $\beta_t\cdot\Gamma(s_t,a_t)$ is an estimation to the LCB-penalty $\Gamma^{\rm lcb}(s_t,a_t)$ in Eq. (\ref{eq:lcb}) for an appropriately selected tuning parameter $\beta_t$.
\end{claim}


The LCB in Eq.~\eqref{eq:lcb} can only be applied in linear cases. We refer to \S\ref{sec:theore} for a detailed explanation and proof. Intuitively, the ensemble $Q$-functions estimates a non-parametric $Q$-posterior \citep{boot-2016, boot-2018}. Correspondingly, the uncertainty quantifier $\Gamma(s_t,a_t)$ estimates the standard deviation of the $Q$-posterior, which scales with the LCB-penalty in linear MDPs. As an example, we show that under the tabular setting, $\Gamma^{\rm lcb}(s_t,a_t)$ is approximately proportional to the reciprocal pseudo-count of the corresponding state-action pair in the dataset. In offline RL, such uncertainty quantification measures how trustworthy the value estimations on state-action pairs are. A low LCB-penalty (or high pseudo-count) indicates that the corresponding state-action pair aligns with the support of offline data.

In the generic MDP settings, we estimate the uncertainty based on the non-parametric ensemble $Q$ networks. Let us consider a scenario where we share a dataset $\cD_j$ to the $i$-th task and obtain $\cD_{j\rightarrow i}$. Correspondingly, we solve the offline RL with the joint dataset $\hD_i=\cD_i\cup \cD_{j\rightarrow i}$. The parameter of the $Q$-function $Q_i$ learned from $\hD_i$ can be solved in a closed form following the least squares value iteration algorithm (LSVI), which iteratively minimizes the following loss function,
\begin{equation}
\label{eq:UTDS-share-lsvi-main}
\tw_{ij} = \min_{w\in \RR^d}
\sum_{k=1}^{|\hD_i|}\bigl(\phi(s^{k}_t, a^{k}_t)^\top w - r(s^{k}_t, a^{k}_t)- V_{t+1}(s^{k}_{t+1})\bigr)^2 + \sum\nolimits_{\cD^{\rm ood}_i \cup \cD^{\rm ood}_j}\bigl(\phi(s^{\rm ood}, a^{\rm ood})^\top w - y^{\rm ood}\bigr)^2,
\end{equation}
where the data in the first and second terms are based on the shared dataset and the OOD data, respectively. We use $\cD^{\rm ood}_i \cup \cD^{\rm ood}_j$ to represent the OOD data from the shared dataset for convenience. As discussed in \S\ref{sec:method}, we sample $s^{\rm ood}$ from $\hD_i$, and sample $a^{\rm ood}$ by following the learned policy $\pi_i$. The target $y^{\rm ood}$ defined in Eq.~\eqref{eq:ood-target} does not rely on the relabeled reward. The solution of $\tw_{i,j}$ is
\begin{equation}
\label{eq:w-UTDS-share}
\tw_{ij} = \widetilde\Lambda^{-1}_{ij}\left(\sum\nolimits^{|\hD_i|}_{k=1}\phi(s^{k}_t, a^{k}_t) y^{k}_t + \sum\nolimits_{\cD^{\rm ood}_i \cup \cD^{\rm ood}_j}\bigl(\phi(s^{\rm ood}, a^{\rm ood}) y^{\rm ood}\right),
\end{equation}
where the covariance matrix $\widetilde\Lambda_{ij}$ for the mixed data $\hD_i$ is
\begin{equation}
\begin{aligned}
\widetilde\Lambda_{ij} = (\Lambda_i + \Lambda_i^{\rm ood}) + (\Lambda_j + \Lambda_j^{\rm ood}) := \widetilde\Lambda_i + \widetilde\Lambda_j,
\end{aligned}
\label{eq:lambda-UTDS-share}
\end{equation}
where $\Lambda_i=\sum^{|\cD_i|}_{k=1}\phi(s^{k}_t, a^{k}_t)\phi(s^{k}_t, a^{k}_t)^\top$ is the covariance matrix of $\cD_i$, and $\Lambda_j$ is the covariance matrix of $\cD_{j\rightarrow i}$. Meanwhile, $\Lambda_i^{\rm ood}=\sum_{\cD^{\rm ood}_i}\phi(s^{\rm ood}, a^{\rm ood})\phi(s^{\rm ood}, a^{\rm ood})^\top$ is the covariance matrix of OOD data $\cD_i^{\rm ood}$, and $\Lambda_j^{\rm ood}$ is the covariance matrix of $\cD^{\rm ood}_{j}$. Based on the solution, the following theorem establishes the relationship between the uncertainty in single dataset and shared dataset.

\begin{theorem}\label{thm:uncertainty-d}
For a given state-action pair $(s,a)$, we denote the uncertainty for the single-task dataset $\cD_i$ and shared dataset $\hD_i=\cD_i\cup \cD_{j\rightarrow i}$ as $\Gamma_i(s,a;\cD_i)$ and $\Gamma_i(s,a;\hD_i)$, respectively. Then the following inequality holds as
\begin{equation}\label{eq:lambda-ij}
\Gamma^{\rm lcb}_i(s,a;\cD_i)\geq \Gamma^{\rm lcb}_i(s,a;\hD_i), {\:\:\rm and\:\:} \Gamma_i(s,a;\cD_i)\geq \Gamma_i(s,a;\hD_i),
\end{equation}
where signifies that the shared data reduce the ensemble uncertainty. 
\end{theorem}

We refer to \S\ref{sec:theore} for a proof. Theorem~\ref{thm:uncertainty-d} shows that the uncertainty decreases with more shared data, which is also illustrated in Figure~\ref{fig:ucb-point}. If a $(s,a)$ pair scarcely occurs in the single-task dataset, the LCB penalty $\Gamma_i(s,a)$ will be high and the corresponding $Q_i(s,a)$ function is pessimistic, which makes the agent hardly choose this action. Nevertheless, such pessimism comes from the lack of knowledge of the environment, which does not indicate $a$ is actually a bad choice. In UTDS, the penalty $\Gamma_i(s,a)$ will gradually decrease with more shared $(s,a)$ (or similar) samples, which makes the value function becomes less pessimistic and extends the agent's knowledge of the environment. 

\subsection{Optimality Gap}

Recent theoretical analysis shows that appropriate uncertainty quantification leads to provable efficiency in offline RL \citep{xie2021bellman,PBRL-2022}. In particular, the pessimistic value iteration \citep{pevi-2021} with a general $\xi$-uncertainty quantifier as the penalty achieves provably efficient pessimism in offline RL. In linear MDPs, the LCB-penalty defined in Eq.~\eqref{eq:lcb} is known to be a $\xi$-uncertainty quantifier for the appropriately selected factor $\{\beta_t\}_{t\in[T]}$. In the following, we show that UTDS with the shared dataset $\hD_i$ also forms a valid $\xi$-uncertainty quantifier. We denote $\cT V_{t+1}(s_t,a_t)=r(s_t,a_t)+\gamma \mathbb{E}_{s_{t+1}\sim P_t(\cdot|s_t,a_t)}\big[V_{t+1}(s_{t+1},a_{t+1})\big]$ as the true Bellman target.
\begin{theorem}
\label{thm:UTDS-lcb-share}
Let $\widetilde\Lambda_{ij} \succeq \lambda \cdot \bI$, if we set the OOD target as $y^{\rm ood} = \cT V_{t+1}(s^{\rm ood}, a^{\rm ood})$ for the shared dataset $\hD_i=\cD_i\cup \cD_{j\rightarrow i}$, then it holds for $\beta_t=\cO\bigl(T\cdot \sqrt{d} \cdot \log(T/\xi)\bigr)$ that
\begin{equation}\nonumber
\Gamma^{\rm lcb}_i(s_t,a_t;\hD_i)=\beta_t\big[\phi(s_t,a_t)^\top \widetilde \Lambda_{ij}^{-1}\phi(s_t,a_t)\big]^{\nicefrac{1}{2}}
\end{equation}
forms a valid $\xi$-uncertainty quantifier, where $\widetilde\Lambda_{ij}$ is the covariance matrix given in Eq.~\eqref{eq:lambda-UTDS-share}.
\end{theorem}

We refer to \S\ref{sec:theore} for a proof. Based on Theorem \ref{thm:UTDS-lcb-share}, we further characterize the optimality gap of UTDS based on the pessimistic value iteration as follows.

\begin{corollary}
\label{coro-UTDS-gap}
Under the same conditions as Theorem \ref{thm:UTDS-lcb-share}, for the uncertainty quantification $\Gamma_i(s,a;\cD_i)$ and $\Gamma_i(s,a;\hD_i)$ defined in $\cD_i$ and $\hD_i=\cD_i\cup \cD_{j\rightarrow i}$ respectively, we have
\begin{equation}\label{eq:sub-bound}
{\rm SubOpt} (\pi_i^*, \widetilde\pi_i) \leq \sum\nolimits_{t=1}^{T} \mathbb{E}_{\pi^*_i} \big[ \Gamma_i^{\rm lcb}(s_t,a_t;\hD_i) \big] 
\leq \sum\nolimits_{t=1}^{T} \mathbb{E}_{\pi^*_i} \big[ \Gamma_i^{\rm lcb}(s_t,a_t;\cD_i) \big],
\end{equation}
where $\widetilde\pi_i$ and $\pi^{*}_i$ are the learned policy and the optimal policy in $\hD_i$, respectively.
\end{corollary}


The first inequality can be directly obtained from \citep{pevi-2021} since the covariance matrix $\widetilde{\Lambda}_{ij}$, the learned policy $\widetilde{\pi}_i$ and the optimal policy $\pi_i^*$ are all defined in the shared dataset $\widehat{D}_i$. The second upper bound follows from the fact that $\Gamma_i(s_t,a_t;\hD_i) \leq \Gamma_i(s_t,a_t;\cD_i)$ (Theorem \ref{thm:uncertainty-d}), namely, a larger dataset leads to tighter uncertainty. Corollary \ref{coro-UTDS-gap} shows 
that the optimality gap shrinks if the data coverage of the optimal policy $\pi^*$ is better, in the sense that the expected uncertainty under the optimal policy $\pi^*$ is smaller.

According to Corollary \ref{coro-UTDS-gap}, we remark that having more data does not guarantee a strict decrease in the optimality gap. As an example, let the $i$-th task and the $j$-th task be irrelevant in the sense that the shared dataset $\cD_{j\rightarrow i}$ does not contain any transitions induced by the optimal policy of the task $i$. Then, such data sharing does not help reduce the sub-optimality as the expected uncertainty with respect to the optimal policy $\pi^*_i$ in $\hD_i$ is the same as that in $\cD_i$. Nevertheless, data sharing does not decrease the learning outcome in such a case, which makes UTDS different from previous methods \citep{CDS-2021,UCDS-2022}. In particular, in policy constraint-based methods, sharing data across irrelevant tasks may exacerbate the distribution shift and hinder learning performance. 

\section{Related Work}

\textbf{Offline RL}~~Since the offline RL dataset often has limited coverage, the actions chosen by greedy policy could be OOD actions that are not contained in the dataset, and the corresponding value function typically suffers from a large extrapolation error. This problem is also known as the distribution shift problem \citep{bear-2019}. Previous model-free offline RL algorithms typically rely on policy constraints to restrict the learned policy from producing the OOD actions. In particular, previous works add behavior cloning (BC) loss in policy training \citep{bcq-2019,td3bc-2021,emaq-2021}, measure the divergence between the behavior policy and the learned policy \citep{bear-2019,brac-2019,fisher-2021}, apply advantage-weighted constraints to balance BC and advantages \citep{keep-2020,CRR-2020}, penalize the prediction-error of a variational auto-encoder \citep{anti-2022}, and learn latent actions (or primitives) from the offline data \citep{psla-2020,opal-2021}. Nevertheless, such methods may cause overly conservative value functions and are easily affected by the behavior policy \citep{awac-2020,offline-online-2021}. We remark that the OOD actions that align closely with the support of offline data could also be trustworthy. CQL \citep{cql-2020} directly minimizes the $Q$-values of OOD samples and thus casts an implicit policy constraint. Our method is related to CQL in the sense that both UTDS and CQL enforce conservatism in $Q$-learning. In contrast, we conduct explicit uncertainty quantification for OOD actions, while CQL penalizes the $Q$-values of all OOD samples equally. However, directly adopting such policy constraints to MTDS is challenging due to the possibly diverse behavioral policies from other tasks, yielding loose policy constraints. Another line of research adopts uncertainty-based pessimism to penalize state action pairs with large uncertainty \citep{UWAC-2020,mopo-2020,PBRL-2022,why-so-pess}. Our algorithm is inspired by the uncertainty-based pessimism methods in single-task offline RL.

In contrast with model-free algorithms, model-based algorithms learn the dynamics model directly with supervised learning. Similar to our work, MOPO \citep{mopo-2020} and MOReL \citep{morel-2020} incorporate ensembles of dynamics-models for uncertainty quantification, and penalize the value function through pessimistic updates. Other than the uncertainty quantification, previous model-based methods also attempt to constrain the learned policy through BC loss \citep{deploy-2020}, advantage-weighted prior \citep{mabe-2021}, CQL-style penalty \citep{combo-2021}, and Riemannian submanifold \citep{gelato-2021}. Decision Transformer \citep{transformer-2021} builds a transformer-style dynamic model and casts the problem of offline RL as conditional sequence modeling. However, such model-based methods may suffer from additional computation costs and may perform suboptimally in complex environments \citep{pets-2018,mbpo-2020}. In contrast, UTDS conducts model-free learning and is less affected by such challenges.

\textbf{Data Sharing in RL}~~Previous data sharing methods focus on solving online multi-task and multi-goal challenges. The intuition is using experiences from relevant tasks to solve the given task, where the relevant tasks are typically selected by human knowledge \citep{mt-opt} or hindsight inverse RL \citep{hindsight-1,hindsight-2}. Direct data sharing on offline RL fails due to distribution shift. CDS \citep{CDS-2021} designs the selection criteria and selects the data relevant to the main task. CDS-Zero \citep{UCDS-2022} removes the relabeling process by setting the shared reward to zero to reduce the bias in the reward. The CDS-based methods discard a large amount of data and can be computationally inefficient, since the selection changes with the update of value functions. Other methods also study data sharing in meta-RL \citep{meta-rl} and dynamics adaptation \citep{DARA-2022,ball2021augmented}, where the main and shared tasks have different dynamics. On the contrary, we study the data sharing among tasks from the same domain. 

\textbf{Uncertainty Quantification}~~Our method is related to the previous online RL exploration algorithms based on uncertainty quantification, including bootstrapped $Q$-networks \citep{ob2i-2021,db}, ensemble dynamics \citep{plan2explore-2020}, Bayesian NN \citep{ube-2018,bayesian-2018}, and distributional value functions \citep{dis-2019,ids-2019,bai2022monotonic,shi2023robust}. Uncertainty quantification is more challenging in offline RL than its online counterpart due to the limited coverage of offline data and the distribution shift of the learned policy. In model-based offline RL, MOPO \citep{mopo-2020} and MOReL \citep{morel-2020} incorporate ensemble dynamics-model for uncertainty quantification. BOPAH \citep{lee2020batch} \tcb{and MOOSE \citep{swazinna2021overcoming}} combine uncertainty penalization and behavior-policy constraints. In model-free offline RL, UWAC \citep{UWAC-2020} adopts dropout-based uncertainty \citep{dropout-2016} while relying on policy constraints in learning value functions. In contrast, UTDS does not require additional policy constraints. In addition, according to the study in image prediction with data shift \citep{trust-2019}, the bootstrapped uncertainty is more robust to data shift than the dropout-based approach \cite{rorl,cucb}. EDAC \citep{EDAC-2021} is a concurrent work that uses the ensemble $Q$-network. Specifically, EDAC calculates the gradients of each Q-function and diversifies such gradients to ensure sufficient penalization for OOD actions. In contrast, UTDS penalizes the OOD actions through direct OOD sampling and the associated uncertainty quantification.

\section{Experiments}

In experiments, we provide a suite of domains built on DeepMind Control Suite \citep{dmc-2018} and collect multi-task datasets to construct a benchmark for large-scale MTDS. We do not use the standard D4RL \citep{d4rl-2020} environment since D4RL has a collection of single task datasets.

\subsection{Tasks and Datasets}

The environment contains $3$ domains with $4$ tasks per domain, resulting in $12$ tasks in total. (\romannumeral 1)~\textbf{Walker} (\emph{Stand, Walk, Run, Flip}) tasks aim to control a biped in a 2D vertical plane. Different tasks learn different balancing and locomotion skills. (\romannumeral 2)~\textbf{Quadruped} (\emph{Jump, Roll-Fast, Walk, Run}) tasks aim to control a quadruped within a 3D space. Quadruped learns different moving and balancing skills in a 3D space, which is harder than Walker because of the high-dimensional state-action space. (\romannumeral 3)~\textbf{Jaco Arm} (\emph{Reach top left, Reach top right, Reach bottom left, Reach bottom right}) tasks aim to control a 6-DOF robotic arm with a three-finger gripper to move to different positions. 

For each task, we run the TD3 algorithm \citep{td3-2018} to collect five types of datasets (i.e., \emph{random, medium, medium-replay, replay,} and \emph{expert}) by following the standard settings in offline RL. Each dataset contains $10^3$ episodes of interactions. For MTDS, we share the dataset of each task with the other 3 tasks of the same domain. Meanwhile, since all tasks are associated with $5$ different types of data, the data sharing setup leads to a variety of combinations. To reduce the computation burden, we only share the \emph{replay} dataset for each other tasks, which ensures sufficient coverage of the shared dataset. As a result, we have $4*3*5=60$ two-task sharing settings for each domain.

\tcb{We follow the setting of D4RL to collect the multi-task datasets. (\romannumeral1) The `medium' dataset is generated by training a policy online using TD3, early-stopping the training, and collecting samples from this partially-trained policy. (\romannumeral2) The `random' datasets are generated by unrolling a randomly initialized policy. (\romannumeral3) The `medium-replay' dataset consists of recording all samples in the replay buffer observed during training until the policy reaches the `medium' level of performance. (\romannumeral4) The `expert' dataset is generated by an expert policy trained by a TD3 agent. (\romannumeral5) The `medium-expert' dataset is collected by mixing equal amounts of expert demonstrations and suboptimal data, generated via a partially trained policy or by unrolling a uniform-at-random policy.}

\paragraph{Baselines}
We compare the proposed UTDS with several baselines, including (\romannumeral 1)~\textbf{Direct Sharing} (i.e., CQL) that shares all relabeled data from $\cD_j$ and trains the policy through CQL \citep{cql-2020}; (\romannumeral 2)~\textbf{CDS} \citep{CDS-2021} that selects transitions with conservative $Q$-values within the highest $10\%$ quantiles of $\cD_{j\rightarrow i}$, where the conservative $Q$-value is learned by CQL; and (\romannumeral 3)~\textbf{CDS-Zero} \citep{UCDS-2022} that conservatively sets the shared reward to zero in CDS. Since the original work of CDS or CDS-Zero does not release the code, we implement the algorithms in our tasks and datasets based on their original papers. See \S\ref{sec:imple} for implementation details of UTDS and baselines.

\begin{figure}[t]
\centering
\includegraphics[width=1.0\columnwidth]{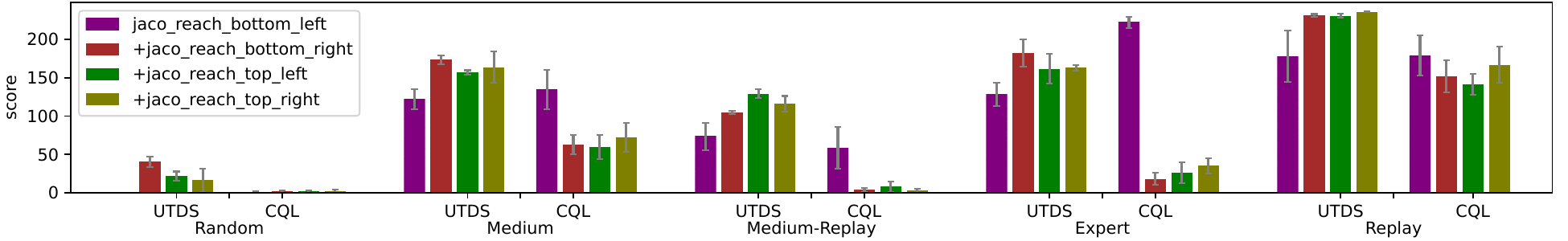}
\caption{The comparison between UTDS and Direct Sharing in \emph{Jaco Arm}. The main task is \emph{Reach-Bottom-Left}, and the shared data are replay datasets from the other three tasks (denoted as `+'). We show results of 5 dataset types of the main task. The shadow bars show the single-task scores.}
\label{fig:share-all-partial}
\end{figure}

\begin{figure*}[t]
\centering
\includegraphics[width=1.0\columnwidth]{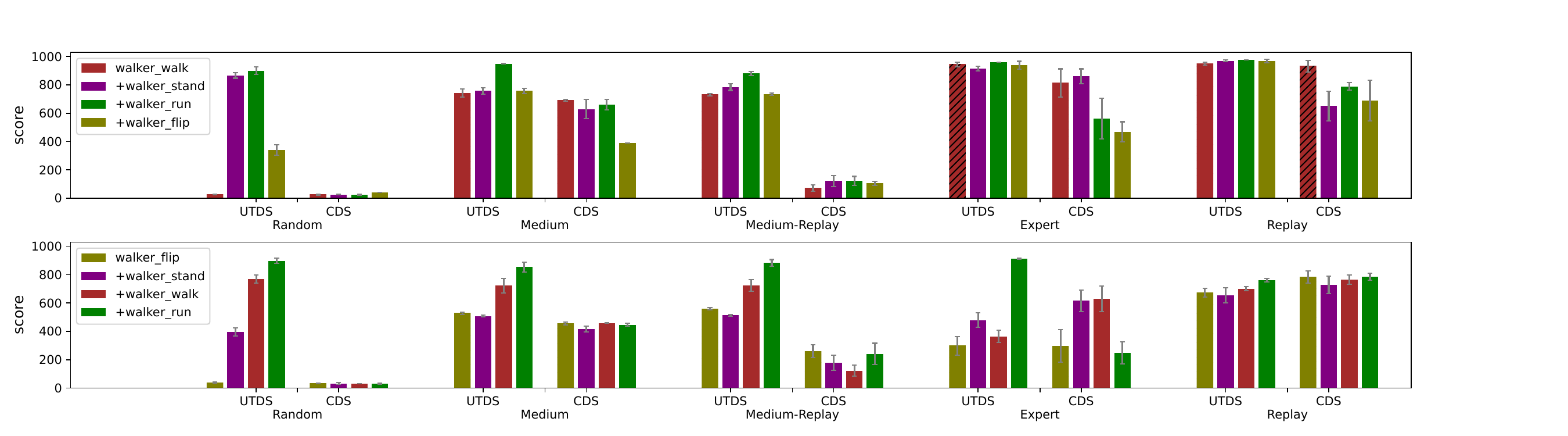}
\caption{The comparison between UTDS and CDS in \emph{Walker}. The main tasks are \emph{Walker-Walk} (Top) and \emph{Walker-Flip} (Bottom), respectively. UTDS generally improves the performance compared to the single-task scores (i.e., the shadow bar), especially for non-expert datasets.}
\label{fig:share-walk-partial}
\end{figure*}

\subsection{Experimental Results}

\paragraph{Directly Sharing in Jaco} Figure~\ref{fig:share-all-partial} shows the result of single-task training, Direct Sharing, and UTDS in \emph{Jaco Arm} domain. Since the different tasks in \emph{Jaco Arm} require moving the robot arm toward different directions, the behavior policies among the tasks differs significantly. The diversity in behavior policies exacerbate the distribution shift in policy constraints, which makes Direct Sharing performs poor in the single-task training. In contrast, the diversity in behavior policies benefit the performance of UTDS. As the result, we observe improved performance of UTDS with data sharing in all settings. We defer the complete results to \S\ref{sec:app-exp}.

\paragraph{Data Sharing in Walker}
We show the result of comparisons between UTDS and CDS in \emph{Walker-Walk} and \emph{Walker-Flip} in Figure~\ref{fig:share-walk-partial}. We find that UTDS generally improves the single-task performance via data sharing. The improvement is significant for non-expert datasets (i.e., \emph{random, medium}, and \emph{medium-replay}) generated by sub-optimal policies, which barely cover the optimal trajectories for the main task. As discussed in Corollary \ref{coro-UTDS-gap}, since the shared data potentially contains trajectories induced by the optimal policy $\pi^*_i$, UTDS achieves a tighter sub-optimality bound $\EE_{\pi^*_i} [\Gamma_i(s,a)]$ and attains better performance than CDS. 

For \emph{expert} and \emph{replay} datasets that cover the optimal trajectories sufficiently well, the improvement 
\begin{wrapfigure}{r}{0.45\textwidth}
\vspace{-15pt}
  \begin{center}
    \includegraphics[width=0.44\textwidth]{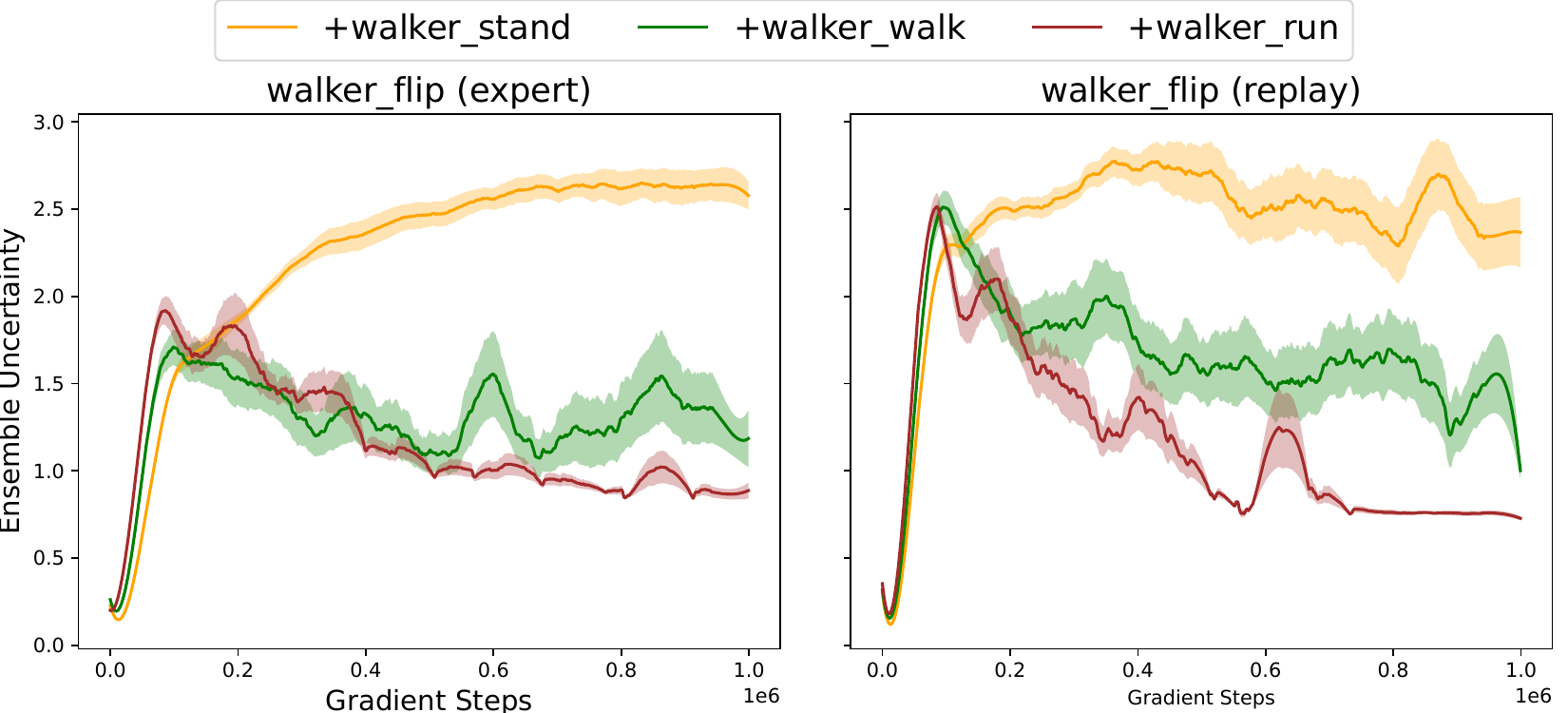}
  \end{center}
  \vspace{-15pt}
  \caption{The illustration of ensemble uncertainty of UTDS calculated on training batches. The main task is \emph{Walker Flip} with \emph{expert} (left) and \emph{replay} (right) datasets.}
  \vspace{-10pt}
  \label{fig:ucb_walker}
\end{wrapfigure}
of data sharing with UTDS is marginal for most cases since it is harder for UTDS to reduce the expected uncertainty $\EE_{\pi^*_i} [\Gamma_i(s,a)]$ with the shared data. The only exception is \emph{Walker-Flip} (\emph{expert} and \emph{replay}) tasks, where sharing the \emph{Walker-Run} dataset improves the performance of UTDS significantly. Intuitively, since the two tasks are closely related, experiences from \emph{Walker-Run} can potentially help the agent in learning complex locomotion skills and improve the learning of \emph{Walker-Flip}. Figure \ref{fig:ucb_walker} shows that data sharing with \emph{Walker-Run} leads to less uncertainty compared to other shared tasks, which results in tighter sub-optimality bound and better performance. We defer the complete results and the comparison of CDS-Zero in \S\ref{sec:app-exp}. We find CDS-Zero performs very similarly to CDS. 

\begin{figure*}[t]
\centering
\includegraphics[width=1.0\columnwidth]{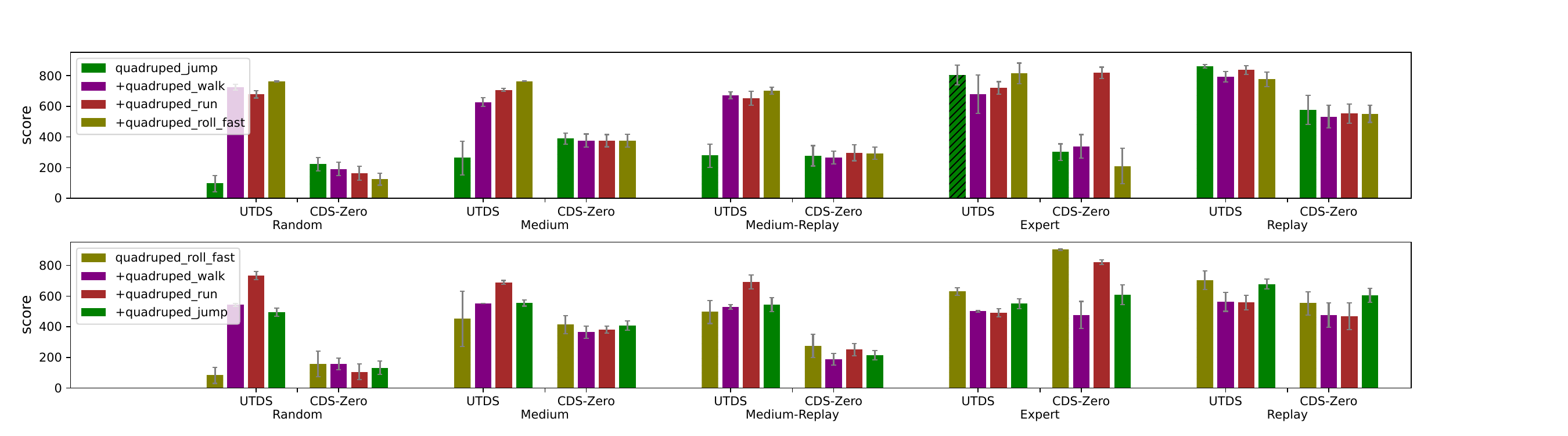}
\vspace{-2em}
\caption{The comparison between the proposed UTDS and CDS-Zero in \emph{Quadruped}. The main tasks are \emph{Quadruped-Jump} (Top) and \emph{Quadruped-Roll-Fast} (Bottom), respectively.}

\label{fig:share-quad-partial}
\end{figure*}

\paragraph{Data Sharing in Quadruped} 
We show the comparison in \emph{Quadruped Jump} and \emph{Quadruped Roll-Fast} in Figure~\ref{fig:share-quad-partial}. We find UTDS consistently improves the performance compared to single-task training in \emph{random, medium}, and \emph{medium-replay} datasets. In contrast, CDS-Zero performs poorly in most sharing tasks due to the significant distribution shift induced by the difference among tasks. 

Nevertheless, we find that the performance of UTDS is slightly decreased when the main task is \emph{expert} or \emph{replay} comparing with the single-task training. We hypothesize that such a phenomenon is also caused by inaccurate feature representation in uncertainty estimation. In particular, our theory requires that the feature $\phi(s,a)$ is accurate for uncertainty quantification, which is not the case in our implementation since we use a deep neural network to learn the representation. Since the \emph{Quadruped} domain has high-dimensional states, the representation learning in the \emph{Quadruped} domain is more challenging. We leave this for future research. The complete results are given in \S\ref{sec:app-exp}. One possible remedy is combining UTDS with representation learning methods such as contrastive learning \citep{ATC-2021} and bootstrapping \citep{BYOL-2020}, which we leave as future research directions. The complete results are given in \S\ref{sec:app-exp}.

\subsection{Visualization of the Uncertainty}

We visualize the change of uncertainty with data sharing in Walker and Quadruped domains, as shown in Fig.~\ref{fig:app-visu-walker} and Fig.~\ref{fig:app-visu-quad}, respectively. We use PCA to project the state-action pairs to two-dimensional space. In both figures, the white dots denote data samples from the main task, and the orange dots denote data sampled from the shared task. We use the ensemble $Q$-networks trained by UTDS on the main task and shared dataset to illustrate the change of uncertainty in data sharing. According to the result, we find (\romannumeral1) in single-task results, the shared data can be considered as OOD data and have large uncertainty; (\romannumeral2) in shared-task results, the uncertainty of the main task data slightly increases with data sharing, and the uncertainties of the shared data are significantly decreased. According to our analysis, the reduced uncertainty makes the value function less pessimistic and also shrinks the sub-optimality bound, thus extending the agent's knowledge of the environment. 

We remark that although the uncertainties quantifier $\Gamma^{\rm lcb}(s,a)=\left[\phi(s,a)^\top\Lambda^{-1}\phi(s,a)\right]^{\nicefrac{1}{2}}$ for any $(s,a)$ pair is guaranteed to reduce with data sharing in linear and tabular MDPs, we find it is somewhat defective. According to the visualization, the uncertainties for some data points from the main task slightly increase. Such effect occurs for several reasons. (\romannumeral1) The linear and tabular MDP settings assume that $\phi(s,a)$ is known and fixed in learning. In contrast, for with high-dimensional problem, $\phi(s,a)$ is randomly initialized by a neural network and is often learned with the value function, which makes the uncertainty estimation inaccurate. (\romannumeral2) We use ensemble $Q$-networks to estimate the uncertainty and prove that the LCB-penalty and ensemble uncertainty are equivalent under mild conditions. In practice, the ensemble uncertainty is accurate with infinite networks, while it is unachievable and we only use 5 ensembles in our experiments. The empirical result shows 5 ensembles are sufficient to provide strong performance. 

\begin{figure*}[h!]
\centering
\includegraphics[width=0.8\columnwidth]{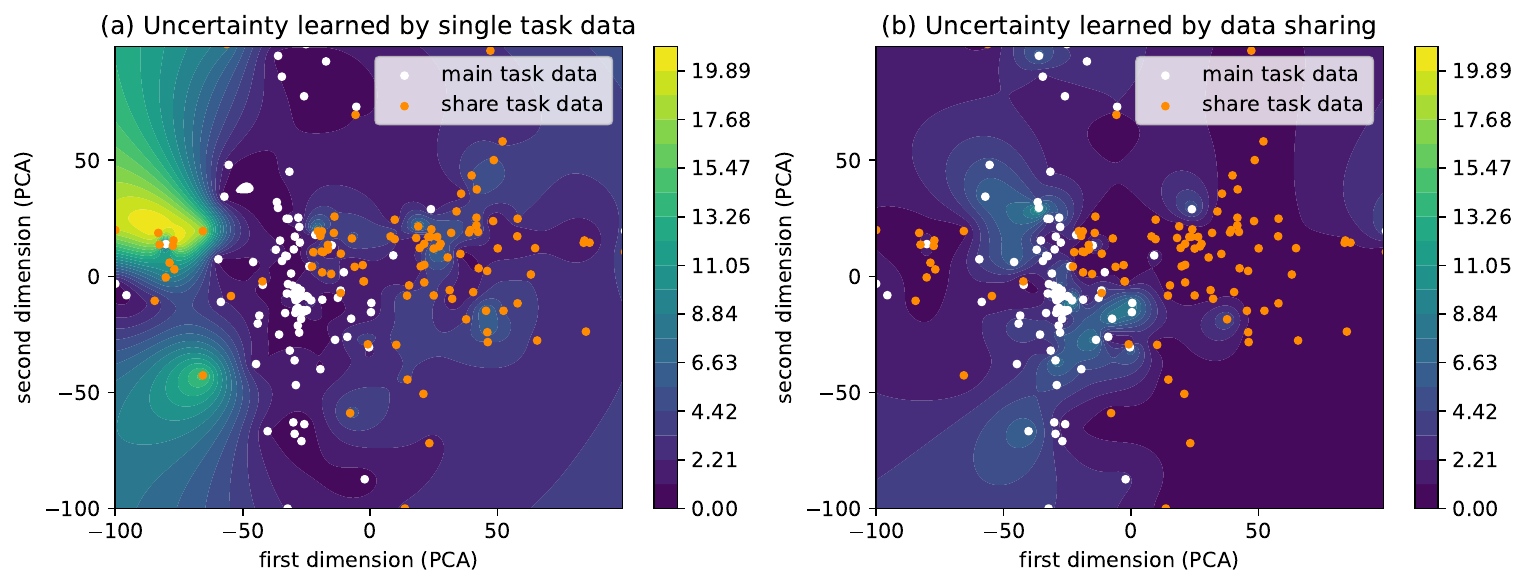}
\caption{Visualization of the uncertainty in data sharing for Walker domain. We share data from Walker-Run (replay) to Walker-Flip (medium-replay) task. (a) The ensemble $Q$-networks are trained in (a) single-task data and (b) shared dataset. We evaluate the uncertainty for both the main task data and the shared data.}
\label{fig:app-visu-walker}
\end{figure*}

\begin{figure*}[h!]
\centering
\includegraphics[width=0.8\columnwidth]{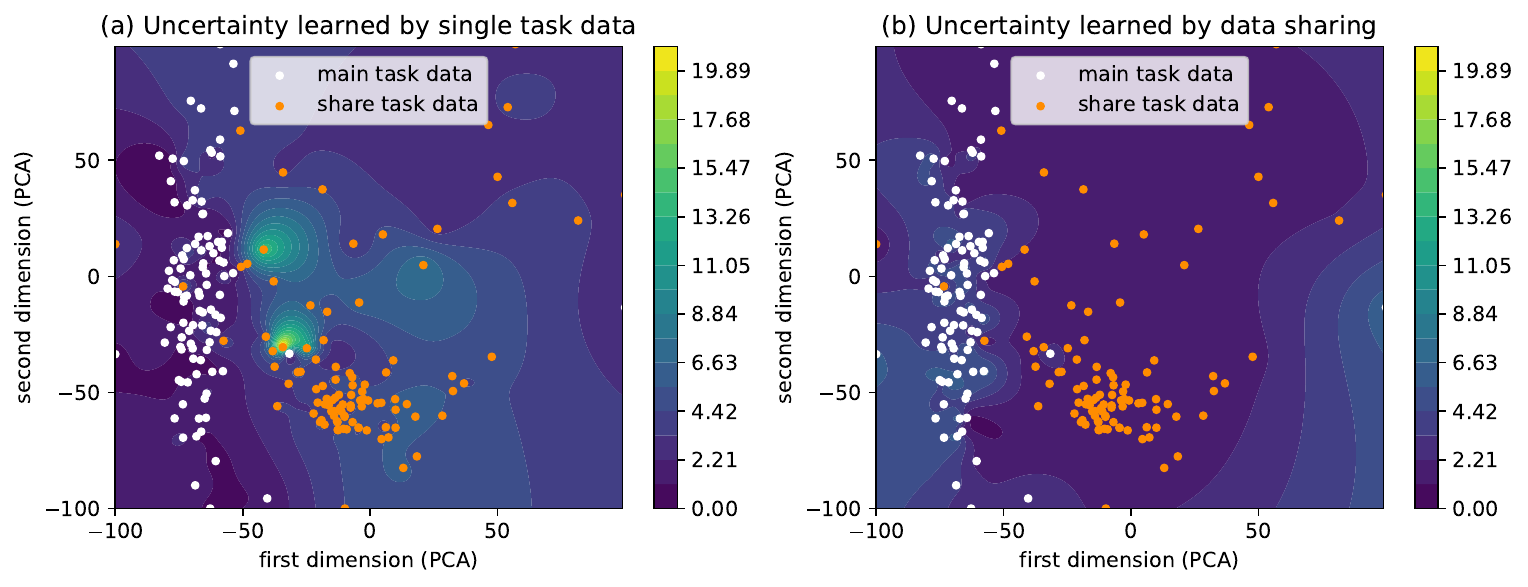}
\caption{Visualization of the uncertainty in data sharing for Quadruped domain. We share data from Quadruped Roll-Fast (replay) to Quadruped Jump (medium) task. The ensemble $Q$-networks are trained in (a) single-task data and (b) shared dataset. We find the uncertainties in for many areas are largely reduced with data sharing.}
\label{fig:app-visu-quad}
\end{figure*}

\section{Conclusion}

In this paper, we propose UTDS algorithm for multi-task offline RL to resolve the distribution shift in data sharing. We apply ensemble $Q$-networks for uncertainty quantification and obtain a pessimistic value function for state-action pairs with low data coverage. Our theoretical result shows that the suboptimality gap of UTDS is related to the expected uncertainty of the shared dataset. Experiments on the proposed benchmark show that UTDS outperforms other data-sharing algorithms in various scenarios. We remark that although UTDS may be slightly weaker in expert or replay dataset, it is of less concern in practical applications as the single-task RL algorithms with uncertainty penalty or policy constraints are sufficient to establish strong performances. 

\section*{Acknowledgments}

This work is supported by the National Natural Science Foundation of China (Grant No.62306242), the National Key R\&D Program of China (Grant No.2022ZD0160100), and Shanghai Artificial Intelligence Laboratory.

\bibliographystyle{elsarticle-num} 
\bibliography{iclr2023_conference}
\clearpage

\appendix

\section{Theoretical Analysis}
\label{sec:theore}

In this section, we give a detailed theoretical analysis of UTDS in linear MDPs \citep{lsvi-2020}. 

\subsection{UTDS for Single-Task Linear MDPs}

In linear MDPs, we assume that the transition dynamics and reward function are linear to the state-action feature embedding $\phi: \cS\times\cA\mapsto \RR^d$, as
\begin{equation}
\label{eq:app-linearMDP}
\mathbb{P}_t(s_{t+1} \,|\, s_t, a_t) = \langle \psi(s_{t+1}), \phi(s_t, a_t) \rangle, \quad r(s_t, a_t) = \theta^\top \phi(s_t, a_t), \quad\forall(s_{t+1}, a_t, s_t)\in\mathcal{S}\times\mathcal{A}\times\mathcal{S},
\end{equation}
where we assume the feature is bounded by $\|\phi\|_2 \leq 1$, and the reward function $r:\cS\times\cA\mapsto[0, 1]$ is also bounded. Since the transition function and reward function are linear to $\phi$, the state-action value function for any policy $\pi$ is also linear in $\phi$, as $Q_t(s_t,a_t)=w_t^\top \phi(s_t,a_t)$. In the following, we omit the notation of time step $t$ in the $Q$-function and the uncertainty. Our discussion is applicable to the solution for arbitrary time step.

For training a specific task $i$ with a single dataset $\cD_i$ in offline RL, we perform Least-Squares Value Iteration (LSVI) to obtain the $Q$-function. The parameter $w_i$ of $Q_i$-function can be solved by minimizing the following loss function as
\begin{equation}
\label{eq:app-lsvi-qlearning}
\hw_i = \min_{w\in \RR^d}\sum_{k=1}^{|\cD_i|}\bigl(\phi(s^k_t, a^k_t)^\top w - r(s^k_t, a^k_t)- V_{t+1}(s^k_{t+1})\bigr)^2 + \lambda\cdot \|w\|_2^2,
\end{equation}
where the experience $(s^k_t, a^k_t, r^k_t, s^k_{t+1})$ is sampled from $\cD_i$, and $V_{t+1}$ is the estimated value function in the $(t+1)$-th step. We denote $y^k_t=r(s^k_t, a^k_t)+ V_{t+1}(s^k_{t+1})$ as the target of LSVI, then the parameter $\hw_i$ can be solved in a closed-form as
\begin{equation}
\label{eq:app-lambda-i}
\hw_i = \Lambda^{-1}_i\sum^{|\cD_i|}_{k=1}\phi(s^k_t, a^k_t) y^k_t,\quad \Lambda_i = \sum^{|\cD_i|}_{k=1}\phi(s^k_t, a^k_t)\phi(s^k_t, a^k_t)^\top + \lambda \cdot \bI,
\end{equation}
where the covariance matrix $\Lambda_i$ accumulates the state-action features from the single-task dataset $|\cD_i|$. The $L_2$-norm of $w_i$ provides regularizations to ensure that $\Lambda_i$ is positive definite. Based on the solution of $w_i$, the action-value function can be estimated by $Q_i(s,a)\approx \hw_i^{\top}\phi(s,a)$. 

In offline RL with linear function assumption, a Lower-Confidence-Bound (LCB) measures the confidence interval of $Q$-function learned by the given dataset. When we learn value functions from the single $\cD_i$, the LCB-penalty for a specific $(s,a)$ pair is given as
\begin{equation}\label{eq:app-lcb}
\Gamma^{\rm lcb}_i(s,a)= \big[\phi(s,a)^\top\Lambda_i^{-1}\phi(s,a)\big]^{\nicefrac{1}{2}},
\end{equation}
which forms an uncertainty quantification with the covariance matrix $\Lambda_i^{-1}$ given the dataset $\cD_i$ \citep{bandit-2011, lsvi-2020, pevi-2021}. Intuitively, $\Gamma^{\rm lcb}_i(s,a)$ can be considered as a reciprocal pseudo-count of the state-action pair in the representation space. 

Without loss of generality, we assume that $w_i$ has a Gaussian prior as $w_i\sim \cN(0,\bI/\lambda)$, and denote the noise in regression as $\epsilon_t^k=y_t^k-w_i^\top \phi(s_t^k,a_t^k)$. Recall that the uncertainty quantification used in UTDS for a specific task $i$ is defined as Eq.~\eqref{eq:uncertainty-gamma} as $\Gamma_i(s,a)=\sqrt{\VV[Q^j_i(s,a)]}$, $j\in[1,\dots,N]$, then the following Lemma established the connection between $\Gamma_i(s,a)$ and the provable efficient LCB-penalty defined in Eq.~\eqref{eq:app-lcb}.

\begin{lemma}[Equivalence between LCB-penalty and Ensemble Uncertainty] We assume that $\epsilon$ follows the standard Gaussian as $\mathcal N(0, 1)$ given the state-action pair $(s,a)$. It then holds for the posterior of $w_i$ given $\cD_i$ that
\begin{equation}
\VV_{\hw_i}[Q_i(s,a)]=\VV_{\hw_i}\bigl(\phi(s, a)^\top \hw_i \bigr) = \phi(s, a)^\top \Lambda^{-1}_i \phi(s, a), \quad \forall (s, a)\in\cS\times\cA.
\end{equation}
\label{app:lemma-ucb}
\end{lemma}

\begin{proof}
This lemma is similar to Lemma 1 in \citep{PBRL-2022}, and the proof follows standard Bayesian linear regression. Under the assumptions, we have that $y\:|\:(s, a),\hw_i\sim \cN(\hw_i^\top \phi(s, a), 1)$. Since the prior distribution $w_i\sim \cN(0, \bI/\lambda)$. The posterior distribution of $\hw_i$ can be obtained following the Bayesian rule as 
\begin{equation}\label{eq::pf_bayes_rule}
\log p(\hw_i \:|\: \cD_i)=\log p(\hw_i) + \log p(\cD_i \:|\: \hw_i)+{\rm Const.},
\end{equation}
Plugging the prior and likelihood distribution into Eq.~(\ref{eq::pf_bayes_rule}) yields
\begin{equation}
\begin{aligned}
\log p(\hw_i\:|\: \cD_i) &= -\|\hw_i\|^2 / 2 -\sum^{|\cD_i|}_{k=1} \| \hw_i^\top\phi(s^k_t, a^k_t) - y^k_t\|^2/2 + {\rm Const.}\\
&= -(\hw_i - \mu_i)^\top \Lambda^{-1}_i(\hw_i - \mu_i)/2 + {\rm Const.},
\end{aligned}
\end{equation}
where we define $\mu_i = \Lambda^{-1}_i \sum^{|\cD_i|}_{k= 1}\phi(s^k_t, a^k_t) y^k_t$, and $\Lambda_i=\sum_{k=1}^{|\cD_i|}\phi(s_t^k,a_t^k)\phi(s_t^k,a_t^k)^\top+\lambda \cdot \mathrm{\mathbf{I}}$. Then we obtain that $\hw_i = w\:|\: \cD_i \sim \cN(\mu_i, \Lambda^{-1}_i)$. It then holds for all $(s, a)\in\cD_i$ that
\begin{equation}
\VV_{\hw_i}\bigl(Q_i(s, a)\bigr) = \VV_{\hw_i}\bigl(\phi(s, a)^\top \hw_i \bigr) = \phi(s, a)^\top \Lambda^{-1}_i \phi(s,a).
\end{equation}
Then we have $\Gamma_i=\sqrt{\VV[Q_i(s, a)]}=\big[\phi(s,a)^\top\Lambda_i^{-1}\phi(s,a)\big]^{\nicefrac{1}{2}}=\Gamma_i^{\rm lcb}$, which concludes our proof.
\end{proof}

In Lemma~\ref{app:lemma-ucb}, we show that the uncertainty estimated by the standard deviation of the $Q$-posterior is equivalent to the LCB-penalty. LCB can only be applied in linear cases, while the ensemble uncertainty is more general with non-parametric ensemble $Q$-networks to handle problems with high-dimensional states and actions. 

In UTDS, we further include an OOD sampling process in $\hT^{\rm ood}$. Then the optimization objective becomes
\begin{equation}\label{eq:app-UTDS-single}
\tw_i = \min_{w\in \RR^d}\sum_{k=1}^{|\cD_i|}\bigl(\phi(s^k_t, a^k_t)^\top w - r(s^k_t, a^k_t)- V_{t+1}(s^k_{t+1})\bigr)^2 + \sum_{s^{\rm ood},a^{\rm ood}\sim \cD^{\rm ood}_i}\bigl(\phi(s^{\rm ood}, a^{\rm ood})^\top w - y^{\rm ood}\bigr)^2,
\end{equation}
where we denote $\cD^{\rm ood}_i$ as an OOD dataset for $\cD_i$ for simplicity. In practice, we sample $s^{\rm ood}$ from the original dataset $\cD_i$, and sample $a^{\rm ood}$ by following the learned policy $\pi$. Such an OOD sampling process is easy to implement since it only relies on the current policy. In Eq.~\eqref{eq:app-UTDS-single}, we remove the $\cL_2$ regularization $\lambda \|w\|_2^2$ since it is only suitable for linear cases. For deep offline RL, we find the OOD sampling process provides sufficient regularization for the extrapolation behavior of OOD data. The explicit solution of Eq.~\eqref{eq:app-UTDS-single} in LSVI is given as 
\begin{equation}
\tw_i = \widetilde\Lambda^{-1}_i\Big(\sum^{|\cD_i|}_{k=1}\phi(s^k_t, a^k_t) y^k_t + \sum_{\cD^{\rm ood}_i} \phi(s^{\rm ood}, a^{\rm ood}) y^{\rm ood} \Big),
\end{equation}
where the covariance matrix is
\begin{equation}
\label{eq:app-lambda-UTDS-single}
\widetilde\Lambda_i = \sum^{|\cD_i|}_{k=1}\phi(s^k_t, a^k_t)\phi(s^k_t, a^k_t)^\top + \sum_{\cD^{\rm ood}_i} \phi(s^{\rm ood},a^{\rm ood})\phi(s^{\rm ood},a^{\rm ood})^{\top}:= \Lambda_i+ \Lambda^{\rm ood}_i,	
\end{equation}
where we denote the second term that accumulates the OOD features as $\Lambda^{\rm ood}_i$.

\subsection{UTDS for Multi-Task Data Sharing}

In MTDS, considering we share a dataset $\cD_j$ to task $i$ as $\cD_{j\rightarrow i}$, then we have $\hD_i=\cD_i\cup \cD_{j\rightarrow i}$ in training. We denote the parameter of $Q$-function learned in  $\hD_i=\cD_i\cup\cD_{j\rightarrow i}$ as $w_{ij}$, then we have
\begin{equation}\label{eq:app-UTDS-share-lsvi}
\begin{aligned}
&\tw_{ij} = \min_{w\in \RR^d}
\sum_{k_1=1}^{|\cD_i|}\bigl(\phi(s^{k_1}_t, a^{k_1}_t)^\top w - r(s^{k_1}_t, a^{k_1}_t)- V_{t+1}(s^{k_1}_{t+1})\bigr)^2 \\
& + \sum_{k_2=1}^{|\cD_{j\rightarrow i}|}\bigl(\phi(s^{k_2}_t, a^{k_2}_t)^\top w - r(s^{k_2}_t, a^{k_2}_t)- V_{t+1}(s^{k_2}_{t+1})\bigr)^2+
\sum_{\cD^{\rm ood}_i \cup \cD^{\rm ood}_j}\bigl(\phi(s^{\rm ood}, a^{\rm ood})^\top w - y^{\rm ood}\bigr)^2,
\end{aligned}
\end{equation}
where the transitions in the first summation are sampled from the original dataset of task $i$, the transitions in the second summation are sampled from the relabeled dataset $\cD_{j\rightarrow i}$ from task $j$, and the last term are sampled from the OOD data $\cD^{\rm ood}_i$ and $\cD^{\rm ood}_j$. In addition, we remark that we do not need relabeled rewards for $(s^{\rm ood},a^{\rm ood})\sim \cD^{\rm ood}_j$ since the target $y^{\rm ood}$ defined in Eq.~\eqref{eq:ood-target} is a pseudo-target that enforces pessimism based on the current value function and the uncertainty quantification. The calculation of $y^{\rm ood}$ does not rely on the reward function.

Following LSVI, the solution of parameter $\tw_{ij}$ in Eq.~\eqref{eq:app-UTDS-share-lsvi} is
\begin{equation}
\label{eq:app-w-UTDS-share}
\tw_{ij} = \widetilde\Lambda^{-1}_{ij}\left(\sum^{|\cD_i|}_{k_1=1}\phi(s^{k_1}_t, a^{k_1}_t) y^{k_1}_t + \sum^{|\cD_{j\rightarrow i}|}_{k_2=1}\phi(s^{k_2}_t, a^{k_2}_t) y^{k_2}_t + 
\sum_{\cD^{\rm ood}_i \cup \cD^{\rm ood}_j}\bigl(\phi(s^{\rm ood}, a^{\rm ood}) y^{\rm ood}\right),
\end{equation}
where we accumulate all OOD data from $\cD^{\rm ood}_i$ and $\cD^{\rm ood}_j$ in the last term. The covariance matrix $\widetilde\Lambda_{ij}$ for the mixed data $\hD_i$ becomes
\begin{equation}
\begin{aligned}
\quad \widetilde\Lambda_{ij} &= 
\sum^{|\cD_i|}_{k_1=1}\phi(s^{k_1}_t, a^{k_1}_t)\phi(s^{k_1}_t, a^{k_1}_t)^\top +
\sum_{\cD^{\rm ood}_i}\phi(s^{\rm ood}, a^{\rm ood})\phi(s^{\rm ood}, a^{\rm ood})^\top\\
&\qquad \qquad +\sum^{|\cD_{j\rightarrow i}|}_{k_2=1}\phi(s^{k_2}_t, a^{k_2}_t)\phi(s^{k_2}_t, a^{k_2}_t)^\top + 
\sum_{\cD^{\rm ood}_j}\phi(s^{\rm ood}, a^{\rm ood})\phi(s^{\rm ood}, a^{\rm ood})^\top \\
&= \Lambda_i + \Lambda_i^{\rm ood} + \Lambda_j + \Lambda_j^{\rm ood} 
= \widetilde\Lambda_i + \widetilde\Lambda_j,
\end{aligned}
\label{eq:app-lambda-UTDS-share}
\end{equation}
which contains the feature covariance matrix from both task $i$ and task $j$. The definition of $\Lambda_i$ and $\Lambda_i^{\rm ood}$ are give in Eq.~\eqref{eq:app-lambda-UTDS-single}. 

Based on the covariance matrix given in Eq.~\eqref{eq:app-lambda-UTDS-single} and  Eq.~\eqref{eq:app-lambda-UTDS-share}, the following theorem establishes the relationship between the uncertainty estimation in the single dataset $\cD_i$ and the shared dataset $\hD_i=\cD_i\cup\cD_{j\rightarrow i}$.

\begin{theorem*}[Theorem \ref{thm:uncertainty-d} restate]
For a given state-action pair $(s,a)$, we denote the uncertainty for the single-task dataset $\cD_i$ and shared dataset $\hD_i=\cD_i\cup \cD_{j\rightarrow i}$ as $\Gamma_i(s,a;\cD_i)$ and $\Gamma_i(s,a;\hD_i)$, respectively. Then the following inequality holds as
\begin{equation}\label{eq:app-lambda-ij}
\Gamma^{\rm lcb}_i(s,a;\cD_i)\geq \Gamma^{\rm lcb}_i(s,a;\hD_i), {\:\:\rm and\:\:} \Gamma_i(s,a;\cD_i)\geq \Gamma_i(s,a;\hD_i),
\end{equation}
where signifies that the shared data reduces the uncertainty. 
\end{theorem*}

\begin{proof}
In the following, we prove that $\Gamma_i^{\rm lcb}(s,a;\cD_i)\geq \Gamma_i^{\rm lcb}(s,a;\hD_i)$. The original relationship can be obtained by following Lemma~\ref{app:lemma-ucb} that $\Gamma_i(s,a;\cD_i)=\Gamma_i^{\rm lcb}(s,a;\cD_i)$ and $\Gamma_i(s,a;\hD_i)=\Gamma_i^{\rm lcb}(s,a;\hD_i)$. 

For UTDS in the shared dataset $\hD_{i}$, according to the solution of $w_{ij}$ and covariance matrix given in Eq.~\eqref{eq:app-w-UTDS-share} and Eq.~\eqref{eq:app-lambda-UTDS-share}, we have 
\begin{equation}
\widetilde\Lambda_{ij} = \widetilde\Lambda_{i}+\widetilde\Lambda_{j},
\end{equation}
where $\widetilde\Lambda_{j}$ denotes the feature covariance matrix from shared task $j$ with OOD sampling, as $\widetilde\Lambda_{j}=\Lambda_j + \Lambda_j^{\rm ood}$. Since both $\widetilde\Lambda_{ij}$ and $\widetilde\Lambda_{i}$ are positive semi-definite, by following the generalized Rayleigh quotient, we have
\begin{equation}
\frac{\phi^\top \widetilde\Lambda_i^{-1} \phi}{\phi^\top (\widetilde\Lambda_i+\widetilde\Lambda_j)^{-1} \phi} \geq \lambda_{\rm min}\big((\widetilde\Lambda_i+\widetilde\Lambda_j)\widetilde\Lambda_i^{-1}\big) = 
\lambda_{\rm min}\big( \mathrm{I} +\widetilde\Lambda_j \widetilde\Lambda_i^{-1}\big) = 
1 + \lambda_{\rm min}\big(\widetilde\Lambda_j \widetilde\Lambda_i^{-1}\big),
\end{equation}
where $\lambda_{\rm min}(\cdot)$ is the minimum eigenvalue of a matrix. Since $\widetilde\Lambda_{j}$ and $\widetilde\Lambda_{i}^{-1}$ are positive semi-definite matrices, the eigenvalues of $\widetilde\Lambda_j \widetilde\Lambda_i^{-1}$ are non-negative and $\lambda_{\rm min}\big(\widetilde\Lambda_j \widetilde\Lambda_i^{-1}\big) \geq 0$. Then we have 
\begin{equation}
\phi(s,a)^\top \widetilde\Lambda_i^{-1} \phi(s,a) \geq \phi(s,a)^\top (\widetilde\Lambda_i+\widetilde\Lambda_j)^{-1} \phi(s,a) =  \phi(s,a)^\top \widetilde\Lambda_{ij}^{-1} \phi(s,a).
\end{equation}
Then according to the definition of LCB-penalty as $\Gamma_i^{\rm lcb}(s,a;\cD_i)=\big[\phi(s,a)^\top \widetilde\Lambda_i^{-1} \phi(s,a)\big]^{\nicefrac{1}{2}}$, we have 
\begin{equation}
\Gamma_i^{\rm lcb}(s,a;\cD_i)\geq \Gamma_i^{\rm lcb}(s,a;\hD_i).
\end{equation}

Further, if we additionally share dataset $\cD_{k\rightarrow i}$ from task $k$ to task $i$, then the covariance matrix becomes $\widetilde\Lambda_{ijk}=\widetilde\Lambda_i+\widetilde\Lambda_j+\widetilde\Lambda_k$, then we have 
\begin{equation}\label{eq:app-succeq-lambda}
\widetilde\Lambda_{ijk}\succeq \widetilde\Lambda_{ij} \succeq \widetilde\Lambda_{i},
\end{equation}
and the corresponding LCB-penalties have the following relationship, as
\begin{equation}\label{eq:app-lcb-ij-ijk}
\phi(s,a)^\top \widetilde\Lambda_i^{-1} \phi(s,a) \:\geq\: \phi(s,a)^\top \widetilde\Lambda_{ij}^{-1} \phi(s,a) \:\geq\: \phi(s,a)^\top \widetilde\Lambda_{ijk}^{-1} \phi(s,a).
\end{equation}
Then we have 
\begin{equation}
\Gamma_i^{\rm lcb}(s,a;\cD_i) \:\geq\: \Gamma_i^{\rm lcb}(s,a;\cD_i \cup \cD_{j\rightarrow i}) \:\geq\: \Gamma_i^{\rm lcb}(s,a;\cD_i \cup \cD_{j\rightarrow i} \cup \cD_{k\rightarrow i}),
\end{equation}
and thus 
\begin{equation}
\Gamma_i(s,a;\cD_i) \:\geq\: \Gamma_i(s,a;\cD_i \cup \cD_{j\rightarrow i}) \:\geq\: \Gamma_i(s,a;\cD_i \cup \cD_{j\rightarrow i} \cup \cD_{k\rightarrow i})
\end{equation}
by Lemma~\ref{app:lemma-ucb}, which concludes our proof. 
\end{proof}

Theorem~\ref{thm:uncertainty-d} shows that the uncertainty for a specific $(s,a)$ pair will decrease with more shared data, which has also been illustrated in Figure~\ref{fig:ucb-point}. For example, if a $(s,a)$ pair is scarcely occurred in the original dataset $\cD_i$, the uncertainty-based penalty will be high and the agent will hardly choose this action since the corresponding $Q_i(s,a)$ function is pessimistic. However, such pessimism comes from the insufficient knowledge of the environment (i.e., epistemic uncertainty), which does not represents $a$ is actually a bad choice in state $s$. In MTDS, the uncertainty $\Gamma_i(s,a)$ will gradually decrease with more $(s,a)$ pairs shared in $\hD_i$, and the value function may become less pessimistic. Such a property is important since the agent will extend its knowledge in the state-action space with data sharing. In addition, since the neural network has generalization ability around near the input, $\Gamma_i(s,a)$ also decreases when we share similar state-action pairs of $(s,a)$. 

In extreme cases, if the shared data contains sufficient transitions in the whole state-action space, the covariance matrix $\widetilde\Lambda_i$ is full rank and $\lambda_{\min}(\Lambda_i)\rightarrow\infty$. Then the uncertainty quantification $\phi(s,a)^\top \widetilde\Lambda_i^{-1} \phi(s,a)\rightarrow 0$. This case is similar to online RL with sufficient exploratory data, and the agent will choose actions based on the estimated value function without uncertainty penalty. 

\subsection{$\xi$-Uncertainty Quantifier}

For offline RL, we learn a pessimistic value function $\hQ_i(s_t,a_t)$ by penalizing the $Q_i$-function with the uncertainty quantification $\Gamma^{\rm lcb}_i(s_t,a_t)$, as
\begin{equation}
\label{eq:app-lsvi-penalty}
\widehat{Q}_i(s_t,a_t) = Q_i(s_t,a_t)-\Gamma^{\rm lcb}_i(s_t,a_t) = \tw^{\top}\phi(s_t,a_t)-\Gamma^{\rm lcb}_i(s_t,a_t),
\end{equation}
where the weight $\tw$ can be $\tw_i$ or $\tw_{ij}$ based on different datasets. Under the linear MDP setting, such pessimistic value iteration is known to be information-theoretically optimal \citep{pevi-2021}. 
In UTDS, we implement this pessimistic value function via uncertainty penalty in $\hT_{\rm UTDS}$ and $\hT_{\rm ood}$.

%

From the theoretical perspective, an appropriate uncertainty quantification is essential to provable efficiency in offline RL \citep{xie2021bellman,xie2021policy,PBRL-2022}. Definition \ref{def1} \citep{pevi-2021} defines a general $\xi$-uncertainty quantifier pessimistic value iteration as a penalty in Eq.~\eqref{eq:app-lsvi-penalty} and achieves provable efficient pessimism in offline RL. In linear MDPs, the LCB-penalty defined in Eq.~\eqref{eq:app-lcb} is known to be a $\xi$-uncertainty quantifier for appropriately selected $\{\beta_t\}_{t\in[T]}$, as $\beta_t[\phi(s_t,a_t)^\top\Lambda_t^{-1}\phi(s_t,a_t)]^{\nicefrac{1}{2}}$. In the following theorem, we show that the proposed UTDS with the shared dataset $\hD_i=\cD_i\cup \cD_{j\rightarrow i}$ also forms a valid $\xi$-uncertainty quantifier with the covariance matrix $\widetilde\Lambda_{ij}$ given in Eq.~\eqref{eq:app-lambda-UTDS-share}.

\begin{theorem*}(Theorem \ref{thm:UTDS-lcb-share} restate)
Let $\widetilde\Lambda_{ij} \succeq \lambda \cdot \bI$, if we set the OOD target as $y^{\rm ood} = \cT V_{t+1}(s^{\rm ood}, a^{\rm ood})$ for shared dataset $\hD_i=\cD_i\cup \cD_{j\rightarrow i}$, then it holds for $\beta_t=\cO\bigl(T\cdot \sqrt{d} \cdot \log(T/\xi)\bigr)$ that
\begin{equation}\nonumber
\Gamma^{\rm lcb}_i(s_t,a_t;\hD_i)=\beta_t\big[\phi(s_t,a_t)^\top \widetilde \Lambda_{ij}^{-1}\phi(s_t,a_t)\big]^{\nicefrac{1}{2}}
\end{equation}
forms a valid $\xi$-uncertainty quantifier, where $\widetilde\Lambda_{ij}$ is the with the covariance matrix given in Eq.~\eqref{eq:app-lambda-UTDS-share}.
\end{theorem*}

\begin{proof}
The proof follows that of the analysis of PBRL \citep{PBRL-2022} in linear MDPs \citep{pevi-2021}. We define the empirical Bellman operator of UTDS learned in $\hD_i$ as $\tT$, then
\begin{equation*}
\tT V_{t+1}(s_t, a_t) = \phi(s_t, a_t)^\top \tw_{ij},
\end{equation*}
where the parameters $\tw_{ij}$ follows the solution in Eq.~\eqref{eq:app-w-UTDS-share}. Following the $\xi$-uncertainty quantifier defined in Definition~\ref{def1}, we upper bound the difference between the empirical Bellman operator of UTDS and the true Bellman operator as
\begin{equation*}
\cT V_{t+1}(s, a) - \tT V_{t+1}(s, a) = \phi(s, a)^\top (w_t - \tw_{ij}). 
\end{equation*}
Here we define $w_t$ as follows
\begin{equation}
\label{eq::pf_def_wt_pess}
w_t = \theta + \int_{\cS} V_{t+1}(s_{t+1})\psi(s_{t+1}) \text{d} s_{t+1},
\end{equation}
where $\theta$ and $\psi$ are defined in Eq.~\eqref{eq:app-linearMDP}. It then holds that 
\begin{equation} 
\begin{aligned}
\label{eq::pf_ood_0}
\cT V_{t+1}(s, a) -& \widetilde \cT V_{t+1}(s, a) = \phi(s, a)^\top (w_t - \tw_{ij})\\
=& \phi(s, a)^\top w_t - \phi(s, a)^\top \widetilde \Lambda^{-1}_{ij}
\sum^{|\cD_i|}_{k_1=1}\phi(s^{k_1}_t, a^{k_1}_t)\bigl(r(s^{k_1}_t, a^{k_1}_t) + V^i_{t+1}(s^{k_1}_{t+1})\bigr)\\
&- \phi(s, a)^\top \widetilde \Lambda^{-1}_{ij}
\sum^{|\cD_{j\rightarrow i}|}_{k_2=1}\phi(s^{k_2}_t, a^{k_2}_t) \bigl(r(s^{k_2}_t, a^{k_2}_t) + V^i_{t+1}(s^{k_2}_{t+1})\bigr)\\
& - \phi(s, a)^\top \widetilde \Lambda^{-1}_{ij}
\sum_{\cD^{\rm ood}_i \cup \cD^{\rm ood}_j} \phi(s^{\rm ood}, a^{\rm ood}) y^{\rm ood},
\end{aligned}
\end{equation}
where we use the solution of $\tw_{ij}$ in Eq.~\eqref{eq:app-w-UTDS-share}. By the definitions of $\widetilde \Lambda_t$ and $w_t$ in Eq.~(\ref{eq:app-lambda-UTDS-share}) and Eq.~(\ref{eq::pf_def_wt_pess}), respectively, we have
\begin{equation}
\begin{aligned}
\label{eq::pf_ood_1}
\phi(s, a)^\top w_t &= \phi(s, a)^\top \widetilde \Lambda_{ij}^{-1} \widetilde \Lambda_{ij} w_t=\phi(s, a)^\top \widetilde \Lambda_{ij}^{-1} \biggl(
\sum^{|\cD_i|}_{k_1=1}\phi(s^{k_1}_t, a^{k_1}_t) \cT V_{t+1}(s^{k_1}_t, a^{k_1}_t) \\
&+\sum^{|\cD_{j\rightarrow i}|}_{k_2=1}\phi(s^{k_2}_t, a^{k_2}_t) \cT V_{t+1}(s^{k_2}_t, a^{k_2}_t)+ 
\sum_{\cD^{\rm ood}_i \cup \cD^{\rm ood}_j} \phi(s^{\rm ood}, a^{\rm ood}) \cT V_{t+1}(\phi(s^{\rm ood}, a^{\rm ood})) \biggr).
\end{aligned}
\end{equation}

Plugging Eq.~(\ref{eq::pf_ood_1}) into Eq.~(\ref{eq::pf_ood_0}) yields
\begin{equation}\label{eq::pf_ood_2}
\cT V_{t+1}(s, a) - \widetilde \cT V_{t+1}(s, a) = \text{(i)} + \text{(ii)}+\text{(iii)},
\end{equation}
where we define
\begin{align*}
\text{(i)} &=\phi(s, a)^\top \widetilde \Lambda_{ij}^{-1}
\sum^{|\cD_i|}_{k_1=1}\phi(s^{k_1}_t, a^{k_1}_t) \bigl(\cT V_{t+1}(s^{k_1}_t, a^{k_1}_t) - r(s^{k_1}_t, a^{k_1}_t) - V^i_{t+1}(s^{k_1}_{t+1}) \bigr),\\
\text{(ii)} &=\phi(s, a)^\top \widetilde \Lambda_{ij}^{-1}
\sum^{|\cD_{j\rightarrow i}|}_{k_2=1}\phi(s^{k_2}_t, a^{k_2}_t) \bigl(\cT V_{t+1}(s^{k_2}_t, a^{k_2}_t) - r(s^{k_2}_t, a^{k_2}_t) - V^i_{t+1}(s^{k_2}_{t+1}) \bigr), \\
\text{(iii)} &= \phi(s, a)^\top \widetilde \Lambda_{ij}^{-1}\sum_{\cD^{\rm ood}_i \cup \cD^{\rm ood}_j} \phi(s^{\rm ood}, a^{\rm ood})\bigl(\cT V_{t+1}(s^{\rm ood}, a^{\rm ood}) - y^{\rm ood}\bigr).
\end{align*}

Following the standard analysis based on the concentration of self-normalized process \citep{bandit-2011, azar2017minimax, wang2020reward, lsvi-2020, pevi-2021} and the fact that $\widetilde\Lambda_{ij} \succeq \lambda \cdot \bI$, it holds that
\begin{equation}
|\text{(i)}| \leq \beta_i \cdot\big[\phi(s,a)^\top\widetilde\Lambda_i^{-1}\phi(s,a)\big]^{\nicefrac{1}{2}},\quad
|\text{(ii)}| \leq \beta_j \cdot\big[\phi(s,a)^\top\widetilde\Lambda_j^{-1}\phi(s,a)\big]^{\nicefrac{1}{2}},
\end{equation}
with probability at least $1-\xi$, where $\beta_{i(j)} = \cO\bigl(T\cdot \sqrt{d} \cdot \text{log}(T/\xi)\bigr)$. By following Eq.~\eqref{eq:app-lcb-ij-ijk}, we have 
\begin{equation}
\begin{aligned}
|\text{(i)}| + |\text{(ii)}| &\leq 2 \max\{\beta_i,\beta_j\} \cdot\big[\phi(s,a)^\top\widetilde\Lambda_{ij}^{-1}\phi(s,a)\big]^{\nicefrac{1}{2}}  \\
&=\beta_t \cdot\big[\phi(s,a)^\top\widetilde\Lambda_{ij}^{-1}\phi(s,a)\big]^{\nicefrac{1}{2}}.
\end{aligned}
\end{equation}
where we denote $\beta_t = 2 \max\{\beta_i,\beta_j\}$. Meanwhile, by setting $y^{\rm ood} = \cT V_{t+1} (s^{\rm ood}, a^{\rm ood})$, it holds that $\text{(iii)} = 0$. Thus, we obtain from Eq.~\eqref{eq::pf_ood_2} that 
\begin{equation}
|\cT V_{t+1}(s, a) - \tT V_{t+1}(s, a)| \leq \beta_t \cdot\big[\phi(s_t,a_t)^\top\Lambda_{ij}^{-1}\phi(s_t,a_t)\big]^{\nicefrac{1}{2}}
\end{equation}
with probability at least $1-\xi$, which concludes our proof.
\end{proof}

Theorem \ref{thm:UTDS-lcb-share} shows that the disagreement among ensemble $Q$-networks is a valid $\xi$-uncertainty quantifier with data sharing, while it needs the covariance matrix $\widetilde\Lambda_{ij}$ is lower bounded. In practice, such an assumption can be achieved through (\romannumeral1) sharing more data to $\hD_i$ since we have $\widetilde\Lambda_{ij} \succeq \widetilde\Lambda_{i}$ according to Eq.~\eqref{eq:app-succeq-lambda}, and (\romannumeral2) randomly generating OOD actions to make the embeddings of the OOD sample isotropic, which ensures the eigenvalues of the covariate matrix $\Lambda^{\rm ood}$ are lower bounded. 

In addition, since the transition function is unknown for OOD samples, the OOD target $\cT V_{t+1}$ is impossible to obtain in practice as it requires knowing the transition at the OOD datapoint. In practice, if TD error is sufficiently minimized, then $Q_{t+1}(s, a)$ should well estimate the target $\cT V_{t+1}$. Thus, in practice, the targets of OOD data are set to be $y_i^{\rm ood}=Q_i(s^{\rm ood}, a^{\rm ood}) - \Gamma_i(s^{\rm ood},a^{\rm ood})$ by assuming the TD-error is sufficiently minimized.

\subsection{Suboptimality Gap}

Theorem \ref{thm:UTDS-lcb-share} allows us to further characterize the optimality gap based on the pessimistic value iteration \citep{pevi-2021}. In particular, the following corollary holds,

\begin{corollary*}[Corollary \ref{coro-UTDS-gap} restate]
Under the same conditions as Theorem \ref{thm:UTDS-lcb-share}, for the uncertainty quantification $\Gamma_i(s,a;\cD_i)$ and $\Gamma_i(s,a;\hD_i)$ defined in $\cD_i$ and $\hD_i=\cD_i\cup \cD_{j\rightarrow i}$ respectively, we have
\begin{equation}
{\rm SubOpt} (\pi_i^*, \widetilde\pi_i) \leq \sum\nolimits_{t=1}^{T} \mathbb{E}_{\pi^*_i} \big[ \Gamma_i^{\rm lcb}(s_t,a_t;\hD_i) \big] 
\leq \sum\nolimits_{t=1}^{T} \mathbb{E}_{\pi^*_i} \big[ \Gamma_i^{\rm lcb}(s_t,a_t;\cD_i) \big],
\end{equation}
where $\widetilde\pi_i$ and $\pi^{*}_i$ are the learned policy and the optimal policy in $\hD_i$, respectively.
\end{corollary*}

\begin{proof}
The first inequality holds since $\Gamma^{\rm lcb}_i(s_t,a_t;\hD_i)$ forms a valid $\xi$-uncertainty quantifier by Theorem \ref{thm:UTDS-lcb-share}. We refer to \citep{pevi-2021} for detailed proof. 

The second inequality is induced by $\Gamma_i(s_t,a_t;\hD_i) < \Gamma_i(s_t,a_t;\cD_i)$ in Theorem \ref{thm:uncertainty-d} with the fixed optimal policy $\pi^*_i$. The optimality gap is information-theoretically optimal under the linear MDP setup with finite horizon.
\end{proof}

Consider an extreme case in tabular MDPs, where the state action feature $\phi(s,a)$ is a one-hot vector. Then we have $\Gamma^{\rm lcb}(s,a)=1/\sqrt{N(s,a)+\lambda}$, where $N(s,a)$ is the pseudo-count of the $(s,a)$ pair. In this case, even if we share data that are totally different from the main task, the optimality gap $\mathbb{E}_{\pi^*}[\Gamma^{\rm lcb}(s,a)]$ does not increase since the pseudo-count $N(s,a)$ cannot decrease with data sharing. As a result, even if data sharing does not benefit the data coverage of the optimal policy, it does not degrade performance. Such a key factor makes UTDS inherently different from the policy constraint methods that need appropriate data selection to select similar samples compared to the main task.

\subsection{Supplementary lemma}

According to Lemma \ref{app:lemma-ucb} in our manuscript, the uncertainty term is equivalent to $\Gamma^{\rm lcb}(s,a)$ in linear MDPs. The value of $\Gamma^{\rm lcb}(s,a)=[\phi(s,a)^\top\Lambda^{-1}\phi(s,a)]^{\frac{1}{2}}$ only relies on the state-action embedding in the offline datasets. Then we have the following lemma.

\begin{lemma}[contraction mapping]
For state-action pairs from the offline dataset, the UTDS operator $\widehat{T}^{\rm UTDS}$ is a $\gamma$-contraction operator in the $\mathcal{L}_\infty$-norm.
\end{lemma}

\begin{proof}
Let $Q_1$ and $Q_2$ be two arbitrary $Q$ functions. Since $a\in{\rm Support}(\mu(\cdot|s))$, then we have
\begin{equation*}
\begin{aligned}
&\| \widehat{\mathcal{T}}^{\rm UTDS}Q_1 - \widehat{\mathcal{T}}^{\rm UTDS}Q_2 \|_\infty
\\&=
\max_{s,a} \Big| \big( r(s,a) + \gamma \mathbb{E}_{s'\sim P,a'\sim\pi}[Q_1(s',a')-\Gamma^{\rm lcb}(s',a')] \big) \\ &\qquad\qquad -
\big( r(s,a) + \gamma \mathbb{E}_{s'\sim P,a'\sim\pi}[Q_2(s',a')-\Gamma^{\rm lcb}(s',a')\big)] \Big| \\
&= \gamma \max_{s,a}\left| \mathbb{E}_{s'\sim P,a'\sim\pi}\left[Q_1(s',a') - Q_2(s',a') \right]  \right| \le \gamma \max_{s,a} \|Q_1 - Q_2\|_\infty 
= \gamma \|Q_1 - Q_2\|_\infty.
\end{aligned}
\end{equation*}
Hence, for some arbitrarily initialized $Q$-function, it is guaranteed to converge to a unique fixed point by repeatedly applying $\widehat{\mathcal{T}}^{\rm UTDS}$.
\end{proof}

\section{Implementation Details}
\label{sec:imple}

In this section, we provide the hyper-parameters and experimental settings. We refer to \url{https://github.com/Baichenjia/UTDS} for the open-source implementation and the released multi-task datasets. 

\subsection{UTDS}

The implementation of ensemble-$Q$ networks in UTDS is based on PBRL and EDAC, while we use much fewer ensemble networks (i.e., with 5 networks) compared to PBRL (with 10 networks) and EDAC (with 10-50 networks). We use the \emph{same} hyper-parameter settings for data-sharing tasks in all domains. The hyper-parameters are listed in Table \ref{tab:hyper-UTDS}. 

\begin{table}[h!]
\small
  \caption{Hyper-parameters of UTDS}
  \label{tab:hyper-UTDS}
  \centering
  \begin{tabular}{p{0.55\columnwidth}p{0.4\columnwidth}}
    \toprule
    Hyper-parameters & Value\\
    \midrule
    The number of bootstrapped networks $N$          & 5  \\
    Policy network  & FC(256,256,256,action\_dim) with ReLU activations\\
    $Q$-network  & FC(256,256,256,1) with ReLU activations\\
    Target network smoothing coefficient  & $5e-3$ \\
    Discount factor $\gamma$ & 0.99 \\
    Learning rate (policy) & $1e-4$ \\
    Learning rate ($Q$-networks) & $1e-4$ \\
	Optimizer & Adam  \\
	Batch size & 1024 \\
	Number of OOD samples for each state & 3 \\
	Factor $\beta_1$ for uncertainty penalty in offline data & 0.001 \\
	Factor $\beta_2$ for uncertainty penalty in OOD data & 3.0$\rightarrow$ 0.1 (first half), 0.1 (other)\\
	Factor $\alpha$ for exponentially decay of $\beta_2$ with training steps & 0.99995 \\
	Training steps & 1M \\
    \bottomrule
  \end{tabular}
\end{table}

For the critic training, we use different factors (i.e., $\beta_1$ and $\beta_2$) for the offline data and the OOD data. The update of the $Q$-function in offline data relies on an ordinary Bellman target with a constant factor for uncertainty penalty. In contrast, since the OOD data does not have the reward and transition function, the pseudo-target is used for uncertainty penalty based on the current value estimation. According to our analysis, $\beta_2$ should exponentially decay to ensure $Q(s^{\rm ood},a^{\rm ood})$ converges to a fixed point. 

For the actor training, we follow SAC-N \citep{EDAC-2021} by using the minimum of ensemble $Q$-function $\min_{n=1,\dots,N} Q_n$ as the target, which is approximately equivalent to using $Q_n-\beta_0 \cdot {\rm Std}(Q_n)$ with a fixed $\beta_0$. 

\subsection{Baselines}

\paragraph{Direct Sharing} 
For the shared dataset, we relabel the reward function based on the Mujoco `Physics' class of DeepMind Control Suite. Then we combine the relabeled data with the original data to construct a mixed dataset for training. For direct sharing, we use CQL \citep{cql-2020} to train a policy based on the mixed dataset. 
 
\paragraph{CDS and CDS-Zero \citep{CDS-2021,UCDS-2022}} 
These two methods are implemented based on CQL. For each training step, we sample experiences ten times the size of the batch size, and calculate the conservative value function for each state-action pair. Then we choose examples with conservative $Q$-values within the top 10\% of the sampled batch in training. For CDS-Zero, we do not perform reward relabeling and set the shared reward to be zero. We remark that we additional use layer-normalization for the actor-critic network in CDS-based methods since we find it helps CDS perform more stable. In contrast, we do not use layer-normalization for the proposed UTDS algorithm. 

\subsection{Experimental Settings} 

\paragraph{Environments and Training} 
We use 12 tasks from \emph{Walker, Quadruped}, and \emph{Jaco Arm} domains in experiments, as shown in Figure~\ref{fig:app-task}. For \emph{Walker} domain, the observation space has 24 dimensions, and the action space has 6 dimensions. For \emph{Quadruped} domain, the observation space has much higher dimensions (i.e., 78), which contains egocentric state, torso velocity, torso upright, IMU, and force torque. The action space of \emph{Quadruped} has 12 dimensions. For \emph{Jaco Arm}, the observation space has 42 dimensions, which contains information of the arm, hand, and target. The action space has 9 dimensions. 

According to the default settings of DeepMind Control Suite, the episode lengths for Walker and Quadruped domains are set to 1000, and the episode length for Joco domain is set to 250. As a result, the maximum episodic reward for Walker and Quadruped domains are 1000, and for Joco domain is 250. In experiments, we run each algorithm for 1M training steps with 3 random seeds. We evaluate the algorithm for 10 episodes every 1K training steps. 

\begin{figure*}[h!]
\centering
\includegraphics[width=1.0\columnwidth]{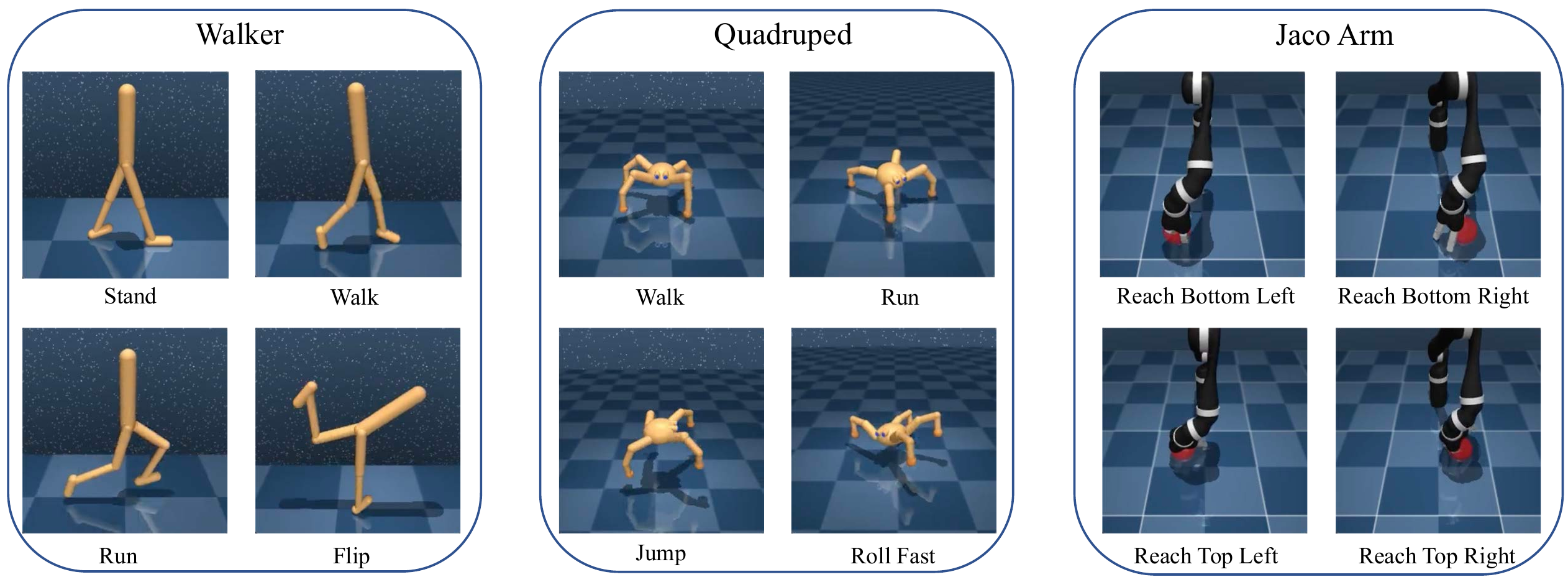}
\caption{The illustration of 12 continuous control tasks from 3 domains in our experiments.}
\label{fig:app-task}
\end{figure*}

\paragraph{Computation Cost Comparison} We compare the computational cost of UTDS, CDS, and CDS-Zero based on the average cost among \emph{Walker} domain. We report the number of parameters, GPU memory, and runtime per sharing task in training. The experiment is conducted on a single RTX-2080Ti GPU. We give the result in Table~\ref{tab:hyper-cost}. The result shows that (\romannumeral1) UTDS requires extra parameters to handle the ensemble of $Q$-networks, while (\romannumeral2) CDS and CDS-Zero need more training time to perform conservative $Q$-value calculation and data-selection per epoch. CDS-Zero is slightly faster since it does not perform reward relabeling. (\romannumeral3) The GPU memories for the two methods are similar since CDS and CDS-Zero need much larger training batches to first perform data selection and then perform RL training. 

\begin{table}[h!]
\small
  \caption{Comparison of computational costs.}
  \label{tab:hyper-cost}
  \centering
  \begin{tabular}{lccc}
    \toprule
    & Runtime (hours / 1M steps) & GPU memory  & Number of parameters \\
    \midrule
    CDS       & 15.75 & 1.31G & 0.37M \\
    CDS-Zero  & 15.68 & 1.31G & 0.37M \\
    UTDS       & 6.56  & 1.34G & 0.81M \\
  \bottomrule
  \end{tabular}
\end{table}



\newpage

\section{Additional Experimental Results}
\label{sec:app-exp}

In this section, we report the complete results of data sharing in \emph{Walker, Quadruped, and Jaco Arm} domains. For each domain, we report 
\begin{itemize}
	\item the single-task training score without data sharing. 
	\item the direct data-sharing score with relabeled rewards, and the policy is trained by CQL. 
	\item the result comparison between the proposed UTDS and CDS \citep{CDS-2021} that performs conservative data selection.
	\item the result comparison between the proposed UTDS and CDS-Zero \citep{UCDS-2022} that performs conservative data selection with zero relabeled reward.
\end{itemize}

\subsection{Single-task results}

UTDS provides a unified view for single-task and multi-task training, as shown in Fig.~\ref{fig:intro}. Thus, UTDS can be directly used for training in single-task datasets. For CDS and CDS-Zero, performing single-task training degenerates into the CQL algorithm \citep{cql-2020} that does not perform data selection. Figure~\ref{fig:app-single-walker}, Figure~\ref{fig:app-single-quad}, and Figure~\ref{fig:app-single-jaco} show the single-task training result for the \emph{Walker, Quadruped, and Jaco Arm} domains, respectively. 

\begin{figure*}[h!]
\centering
\includegraphics[width=0.8\columnwidth]{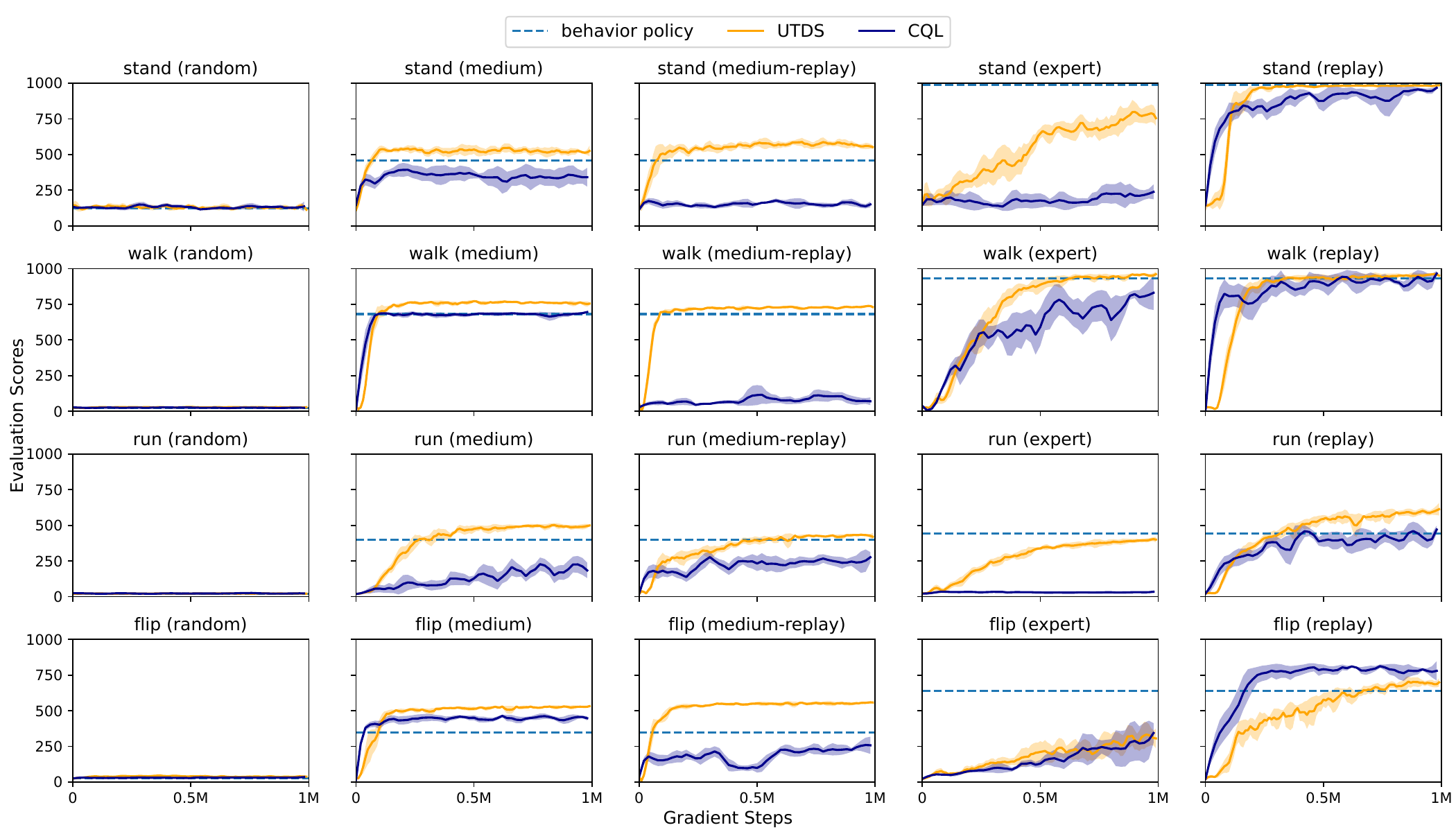}
\caption{Result comparison of single-task training in the \emph{Walker} domain. UTDS outperforms CQL in most tasks, especially for \emph{medium}, \emph{medium-replay}, and \emph{expert} datasets. Both methods perform poorly in random datasets.}
\label{fig:app-single-walker}
\end{figure*}

\begin{figure*}[h!]
\centering
\includegraphics[width=0.8\columnwidth]{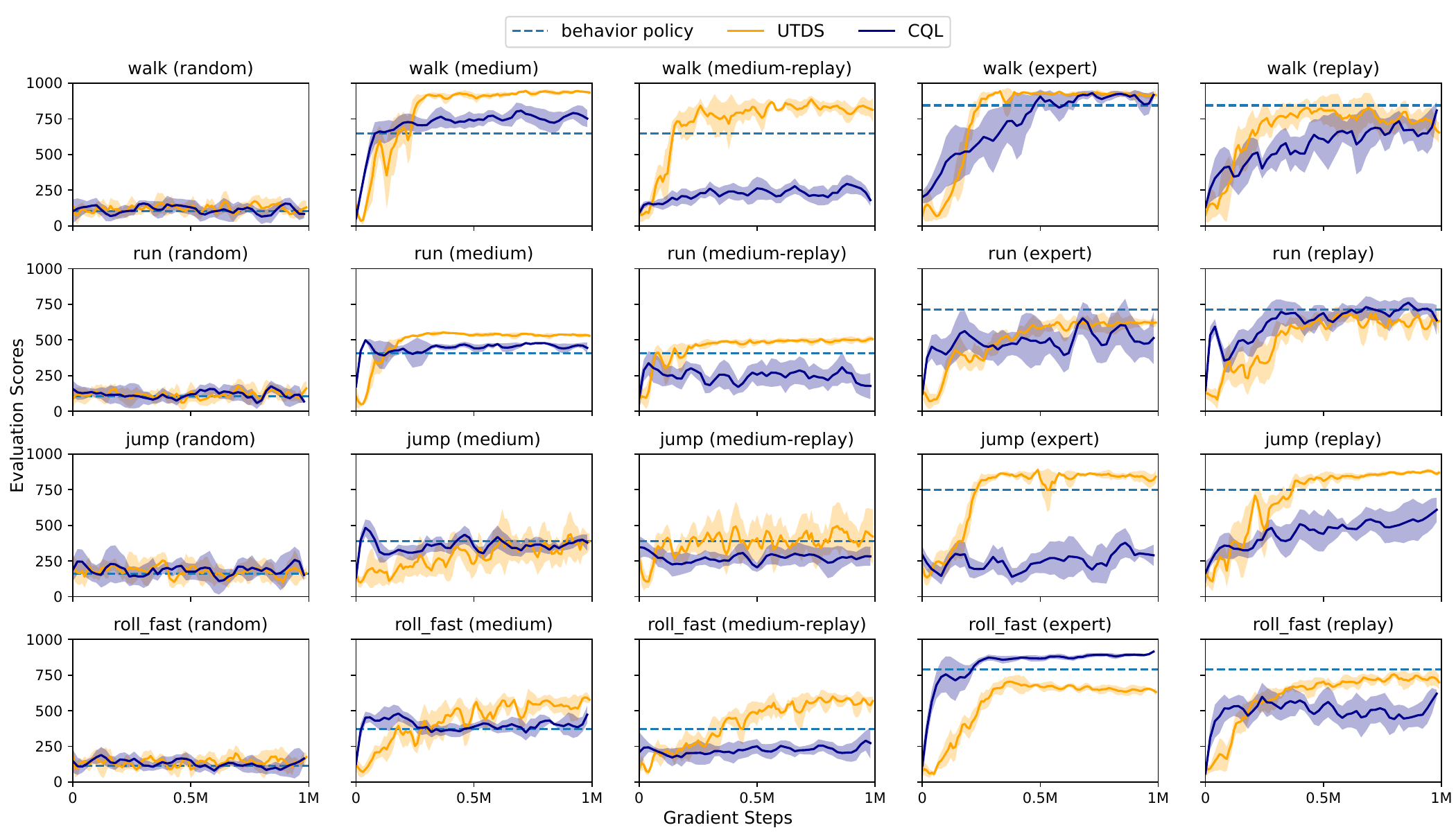}
\caption{Result comparison of single task training in \emph{Quadruped} domain. UTDS outperforms CQL in most tasks, except for the \emph{Roll-Fast} expert dataset. Both methods perform poorly in random datasets.}
\label{fig:app-single-quad}
\end{figure*}

\begin{figure*}[h!]
\centering
\includegraphics[width=0.8\columnwidth]{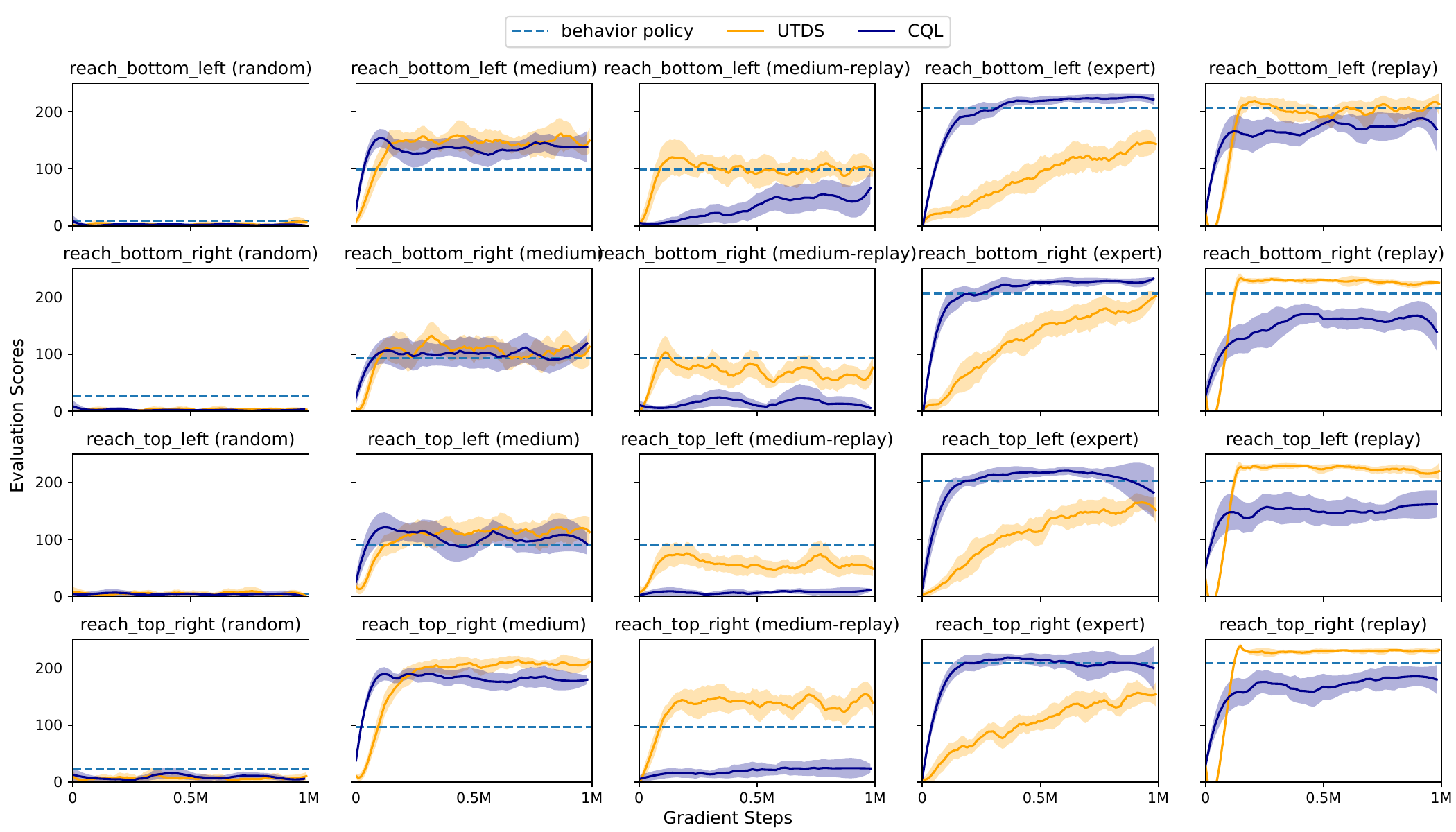}
\caption{Result comparison of single-task training in \emph{Jaco Arm} domain. We find CQL outperforms UTDS in expert dataset. The policy constraints direct imitate the expert policies, which make CQL converge faster with optimal trajectories.}
\label{fig:app-single-jaco}
\end{figure*}

\clearpage
\subsection{Direct data-sharing results}

For offline RL, direct sharing with policy constraint methods can exacerbate the distribution shift when the sharing tasks are very different from the main task. Figure~\ref{fig:app-cql-walker}, Figure~\ref{fig:app-cql-quad}, and Figure~\ref{fig:app-cql-jaco} show the result of the direct sharing for the \emph{Walker, Quadruped,} and \emph{Jaco Arm} domains, respectively. In each figure, we also show the single-task training results for the main task in the leftmost bar (i.e., the shadow bar). 

\begin{enumerate}
\item In \emph{Walker} domain, direct data sharing slightly improves the performance, the reason is the four tasks in \emph{Walker} is closely related. For example, the agent should first \emph{Stand} and then \emph{Walk, Run} and \emph{Flip}, which makes the dataset of \emph{Walk, Run} and \emph{Flip} also contain the experiences of \emph{Stand}. As a result, sharing data from other three tasks often improves the performance of \emph{Stand}.  
\item In \emph{Quadruped} domain, \emph{Jump} task can benefit from data sharing from other tasks since first \emph{Walk}, \emph{Run} or \emph{Roll} make it easier for the agent to \emph{Jump}. Since the behavior policies of \emph{Jump} and other tasks are related, performing direct data sharing also helps learning. For other sharing tasks, direct sharing data degenerates the performance compared to single-task training in most cases.
\item In \emph{Jaco Arm} domain, since the different tasks drive agents to different directions, the behavior policies between tasks can be very different. Direct data sharing significantly exacerbates the distribution shift in offline RL and deteriorates performance. 
\end{enumerate}

We summarize the mean and medium scores for single-task training (with CQL \citep{cql-2020} and PBRL \citep{PBRL-2022}), and direct data sharing with CQL (i.e., CQL-share) in each domain. The results are shown in Table~\ref{tab:res-single}. The result shows direct sharing data degenerates the performance compared to single-task training in most cases. Meanwhile, PBRL outperforms the behavior policy in all domains.

\begin{table}[h!]
\caption{Comparison of single-task training and direct sharing in three domains. The maximum scores of the \emph{Walker} and \emph{Quadruped} domains are 1000, and for \emph{Jaco Arm} is 250.}
\label{tab:res-single}
\centering
\small
\begin{tabular}{c|cc|cc|cc}
\hline
                 & \multicolumn{2}{c|}{Walker} & \multicolumn{2}{c|}{Quadruped} & \multicolumn{2}{c}{Jaco Arm} \\ \hline
                 & Mean         & Median       & Mean           & Median        & Mean          & Median       \\ \hline
Behavior policy & 497.96 & 450.75 & 515.33 & 526.98  & 123.46 & 97.7  \\
Single-task training (CQL) & 363.57       & 267.13       & 437.10         & 411.88        & 106.66        & 128.94       \\ 
Single-task training (PBRL) & 522.17 & 542.33 & 561.75 & 599.58 & 125.43 & 141.37     \\ \hline
Direct share (CQL) & 419.35 & 357.97 & 435.96 & 357.97 & 54.64 & 20.47     \\ \hline
\end{tabular}
\end{table}

\begin{figure*}[h!]
\centering
\includegraphics[width=0.7\columnwidth]{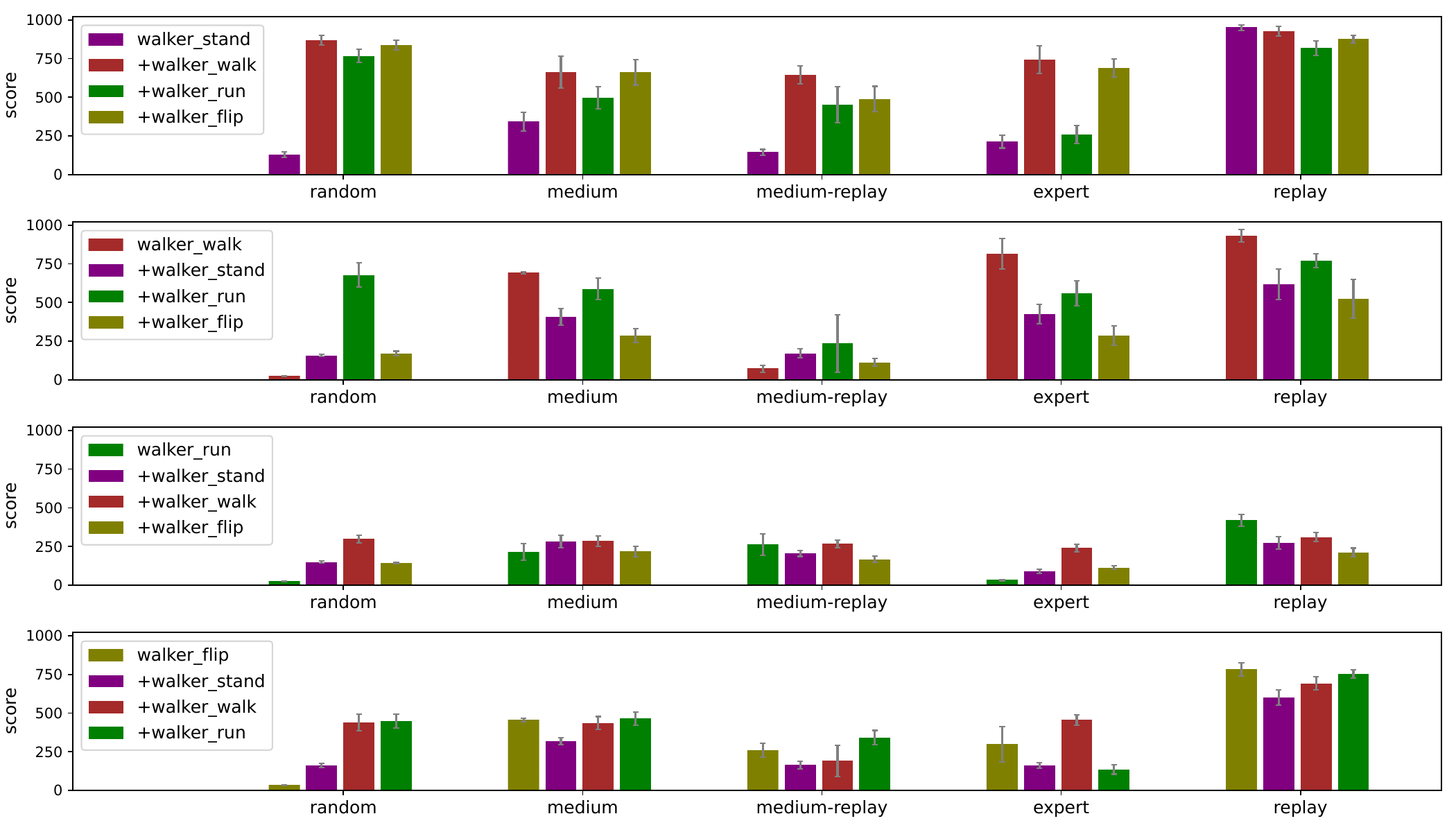}
\caption{Result comparison of single-task training and direct sharing with CQL in \emph{Walker} domain.}
\label{fig:app-cql-walker}
\end{figure*}

\begin{figure*}[h!]
\centering
\includegraphics[width=0.7\columnwidth]{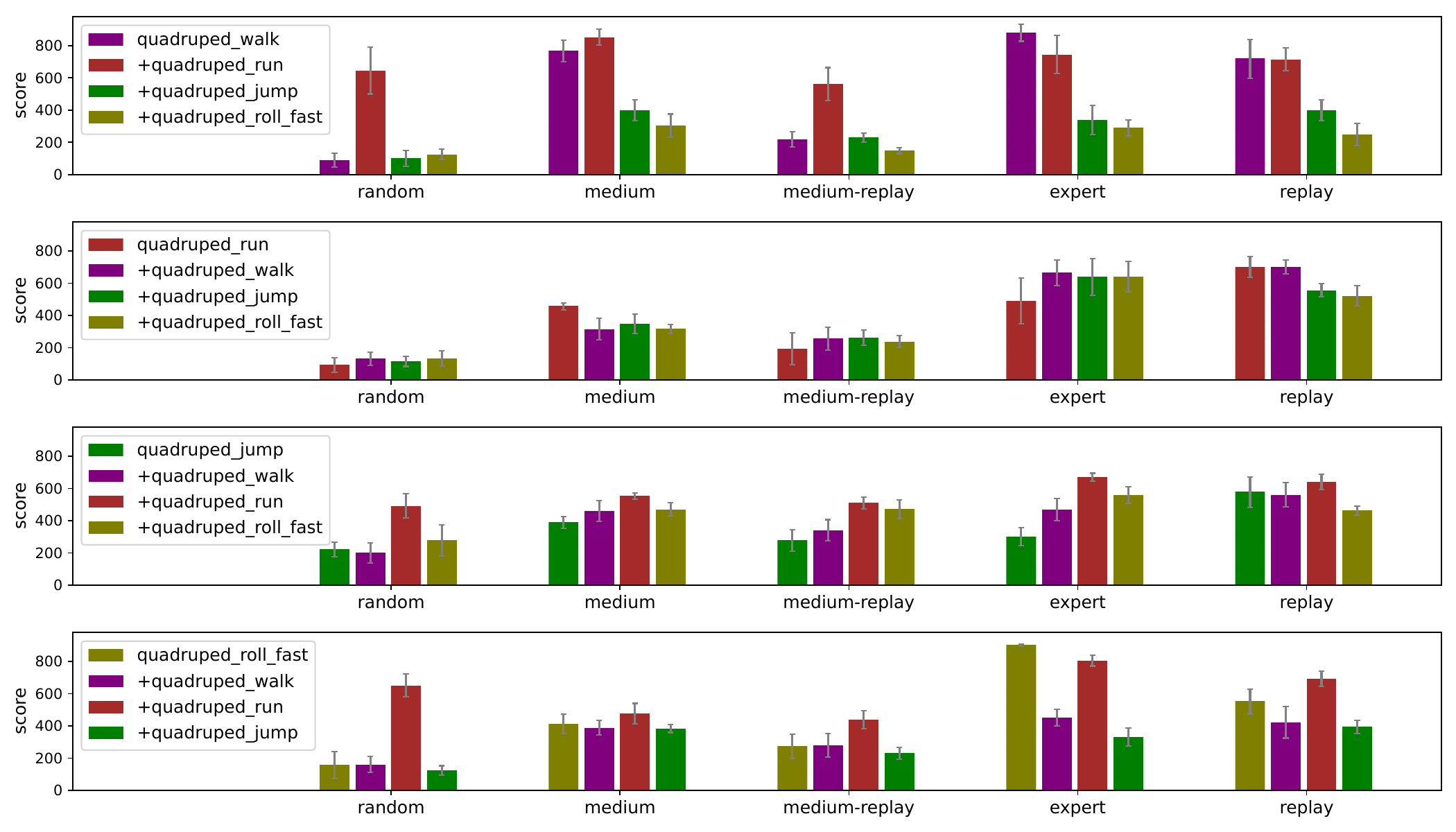}
\caption{Result comparison of single-task training and direct sharing with CQL in \emph{Quadruped} domain.}
\label{fig:app-cql-quad}
\end{figure*}

\begin{figure*}[h!]
\centering
\includegraphics[width=0.7\columnwidth]{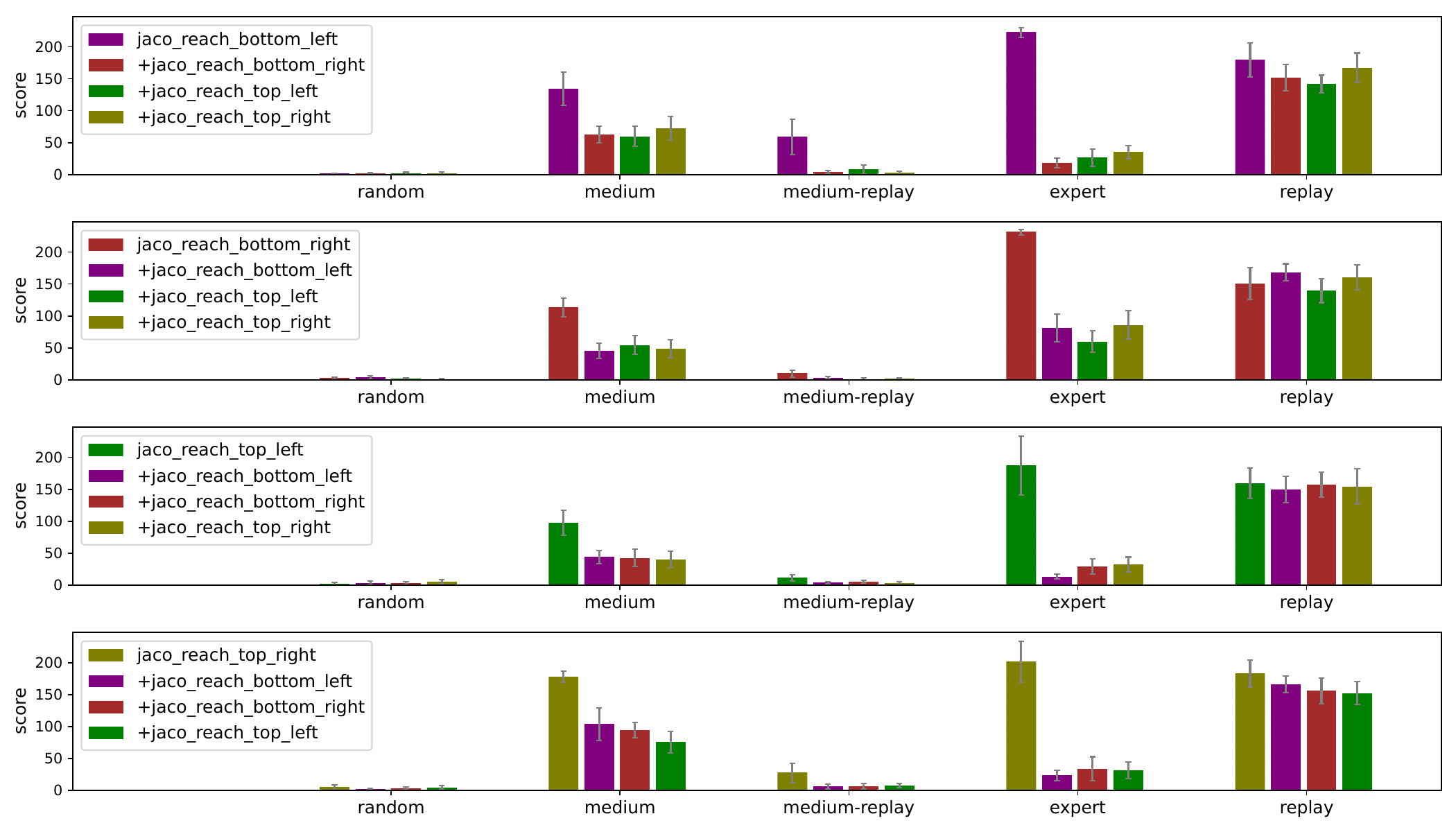}
\caption{Result comparison of single-task training and direct sharing with CQL in \emph{Jaco Arm} domain. Performing direct data sharing with policy constraints methods degrades the performance in most sharing settings.}
\label{fig:app-cql-jaco}
\end{figure*}

\clearpage
\subsection{Result comparison between UTDS and CDS}
\label{sec:app-utds-cds}

We compare the performance of UTDS and CDS in three domains, and the results are shown in
Figure~\ref{fig:app-UTDS-cds-walk}, Figure~\ref{fig:app-UTDS-cds-quad}, and Figure~\ref{fig:app-UTDS-cds-jaco}. In each figure, we show the single-task training results for the main task in the shadow bar.
\begin{enumerate}
\item In \emph{Walker} domain, we find that UTDS generally improves performance through data sharing in the \emph{Stand, Walk,} and \emph{Flip} tasks. An exception is the \emph{Run} task, where data sharing does not significantly improve performance. We hypothesize that the \emph{Run} task is the most difficult task in the \emph{Walker} domain, thus the shared data do not contain state-action pairs of the optimal trajectories in the \emph{Run} task. As a result, sharing data in \emph{Run} task does not bring tighter optimality bound but makes the feature representation learn slower. 
\item In \emph{Quadruped} domain, we find that UTDS generally improves performance in the \emph{Stand, Jump,} and \emph{Roll fast} tasks, especially for non-expert datasets. Similar to \emph{Walker} domain, since \emph{Run} task is more difficult than other tasks, data sharing cannot help reduce the expected uncertainty of trajectories induced by the optimal policy. Compared to UTDS, CDS cannot obtain significant improvement compared to single-task training. 
\item In \emph{Jaco Arm} domain, the behavior policies between tasks can be very different since different tasks drive agents to different directions. CDS based on policy constraints cannot obtain improvements in most cases. In contrast, we find that UTDS improves performance compared to single-task training in most sharing tasks.
\end{enumerate}

\begin{figure*}[h!]
\centering
\includegraphics[width=1.0\columnwidth]{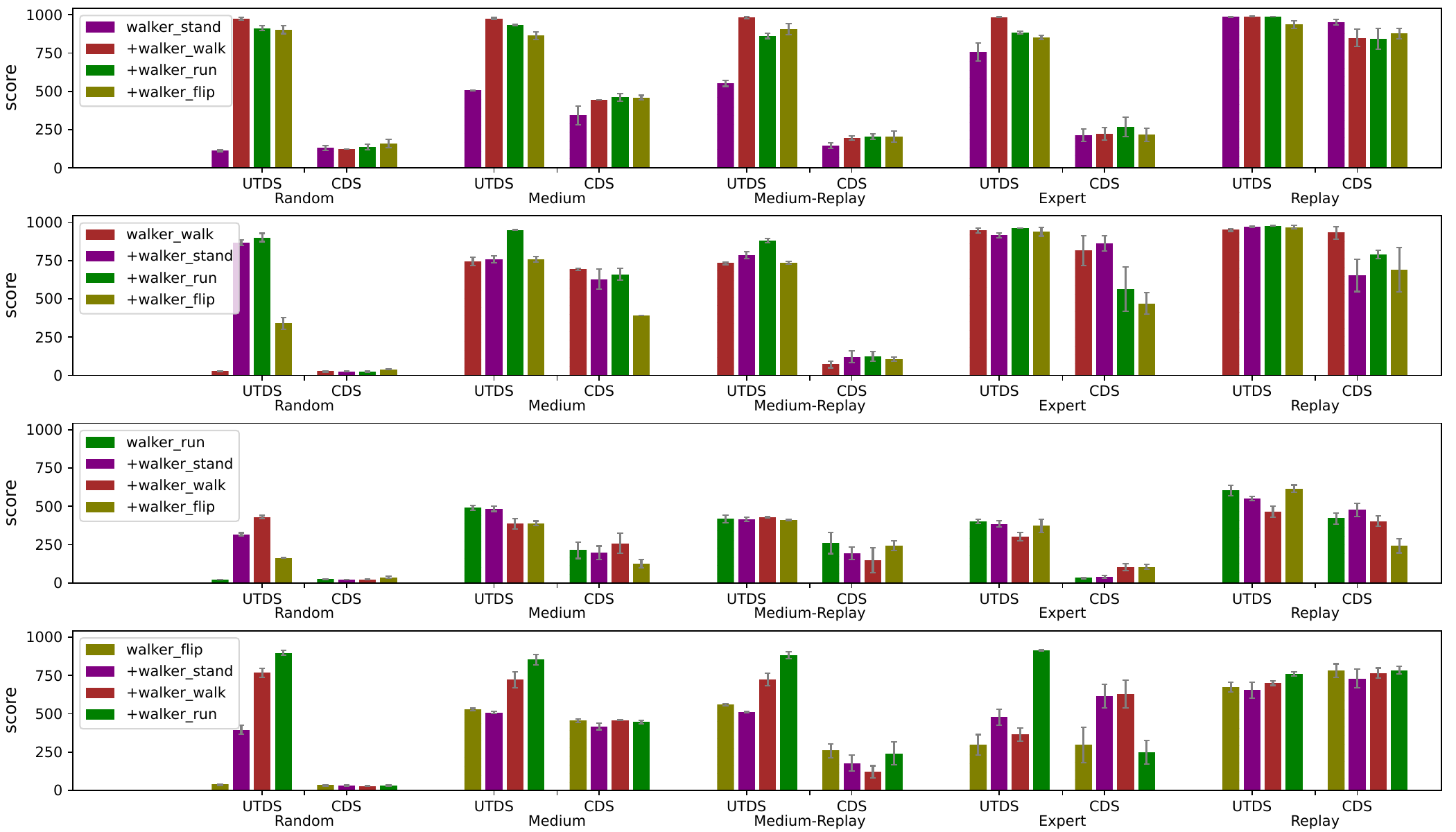}
\caption{Result comparison of the proposed UTDS and CDS in \emph{Walker} domain.}
\label{fig:app-UTDS-cds-walk}
\end{figure*}

\begin{figure*}[h!]
\centering
\includegraphics[width=1.0\columnwidth]{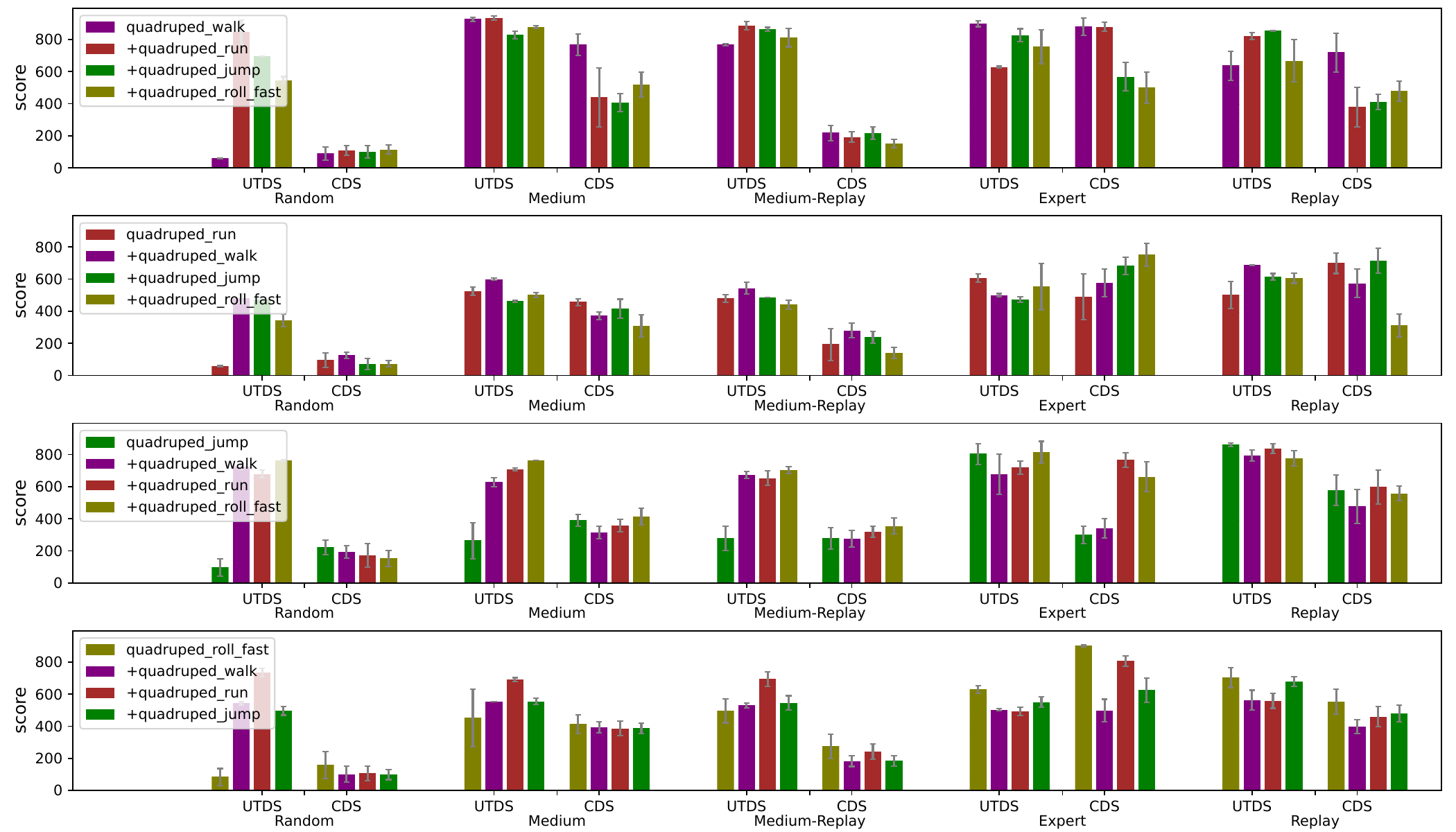}
\caption{Result comparison of the proposed UTDS and CDS in \emph{Quadruped} domain.}
\label{fig:app-UTDS-cds-quad}
\end{figure*}

\begin{figure*}[h!]
\centering
\includegraphics[width=1.0\columnwidth]{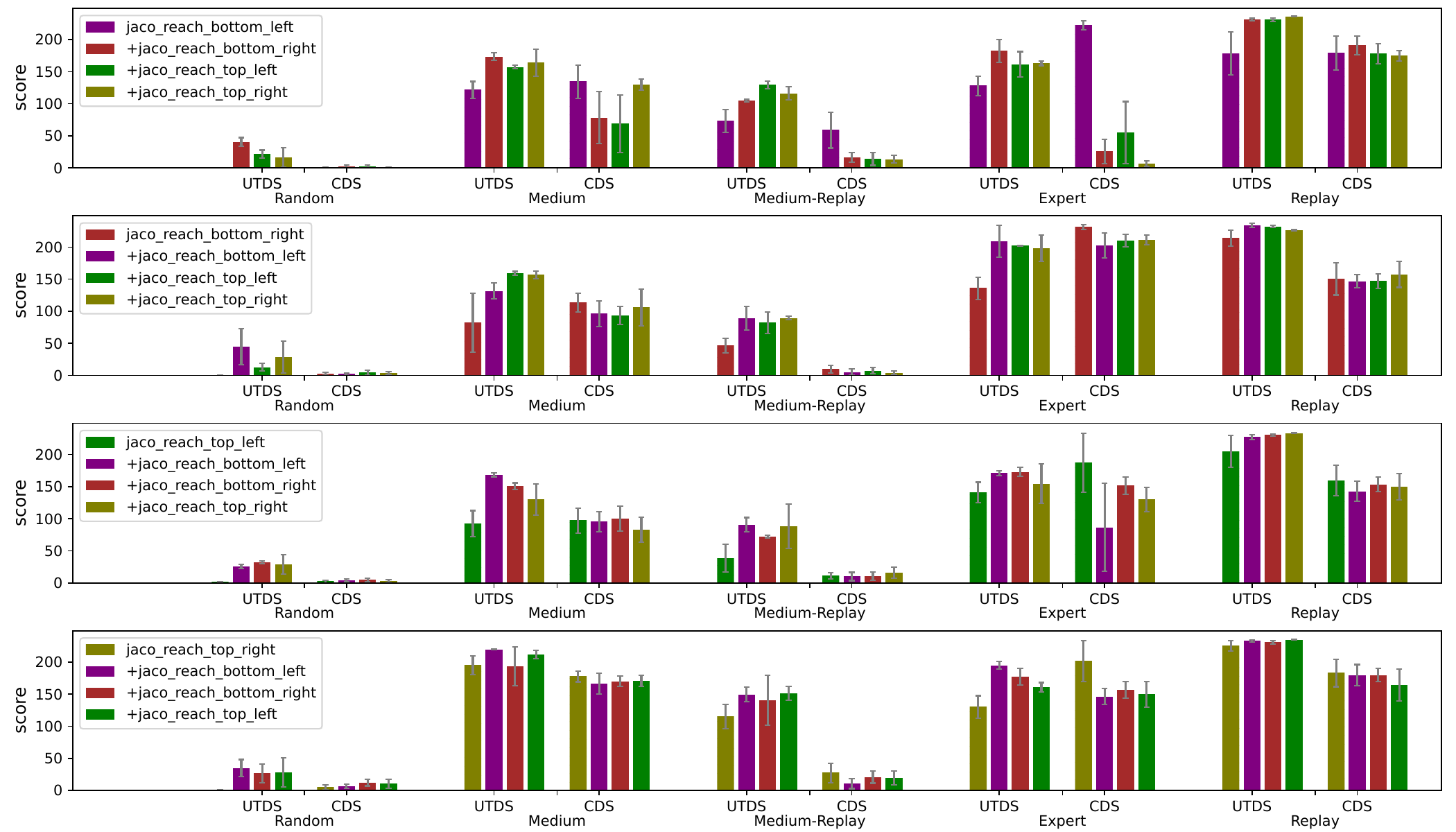}
\caption{Result comparison of the proposed UTDS and CDS in \emph{Jaco Arm} domain.}
\label{fig:app-UTDS-cds-jaco}
\end{figure*}

\clearpage

\subsection{Result comparison between UTDS and CDS-Zero}

We compare the performance of UTDS and CDS-Zero in three domains, and the results are shown in
Figure~\ref{fig:app-UTDSz-cds-walk}, Figure~\ref{fig:app-UTDSz-cds-quad}, and Figure~\ref{fig:app-UTDSz-cds-jaco}. The performance of CDS and CDS-Zero in data sharing are similar. We refer to \S\ref{sec:app-utds-cds} for the analysis.

\begin{figure*}[h!]
\centering
\includegraphics[width=0.9\columnwidth]{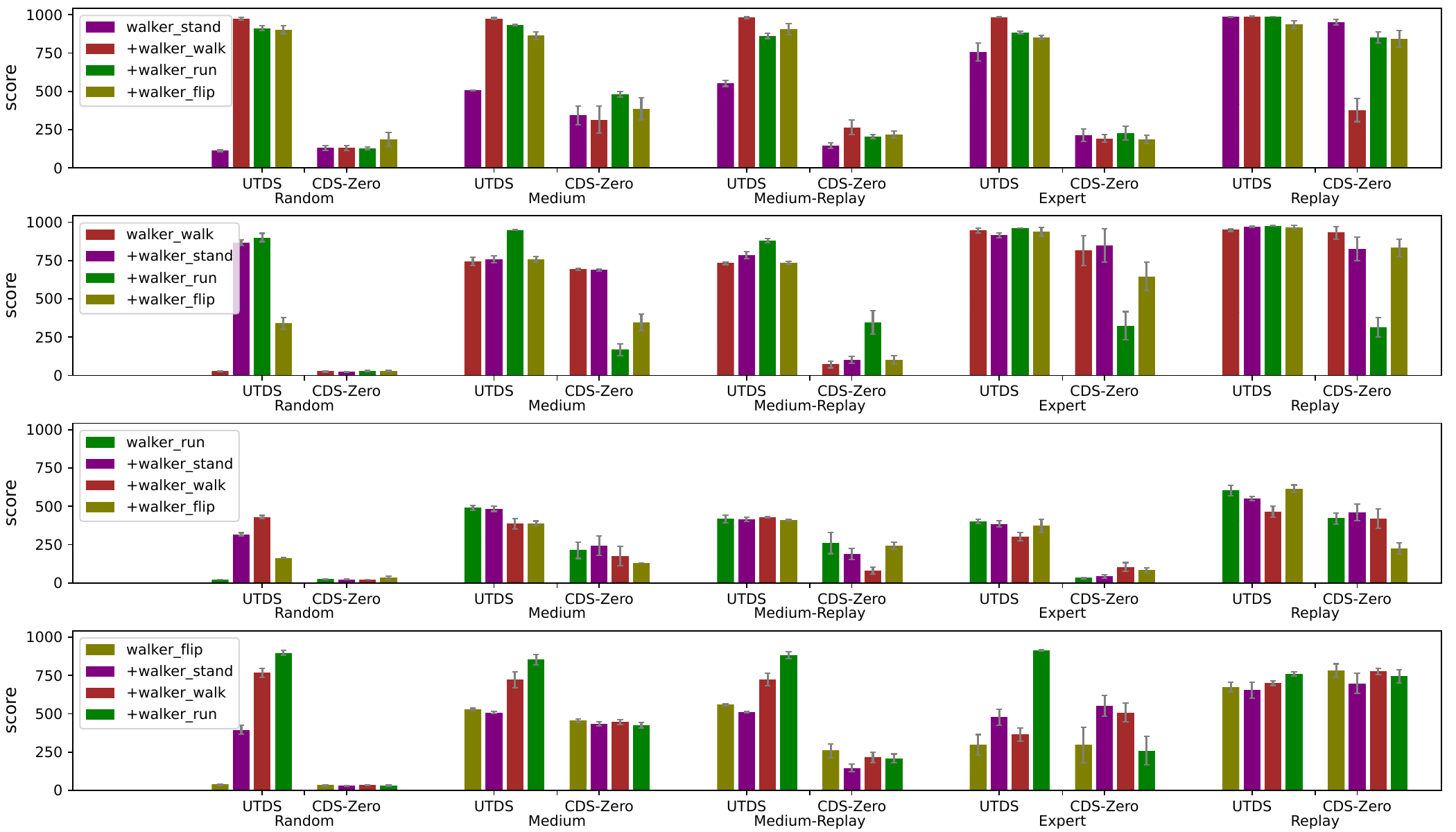}
\caption{Result comparison of the proposed UTDS and CDS-Zero in \emph{Walker} domain.}
\label{fig:app-UTDSz-cds-walk}
\end{figure*}

\begin{figure*}[h!]
\centering
\includegraphics[width=0.9\columnwidth]{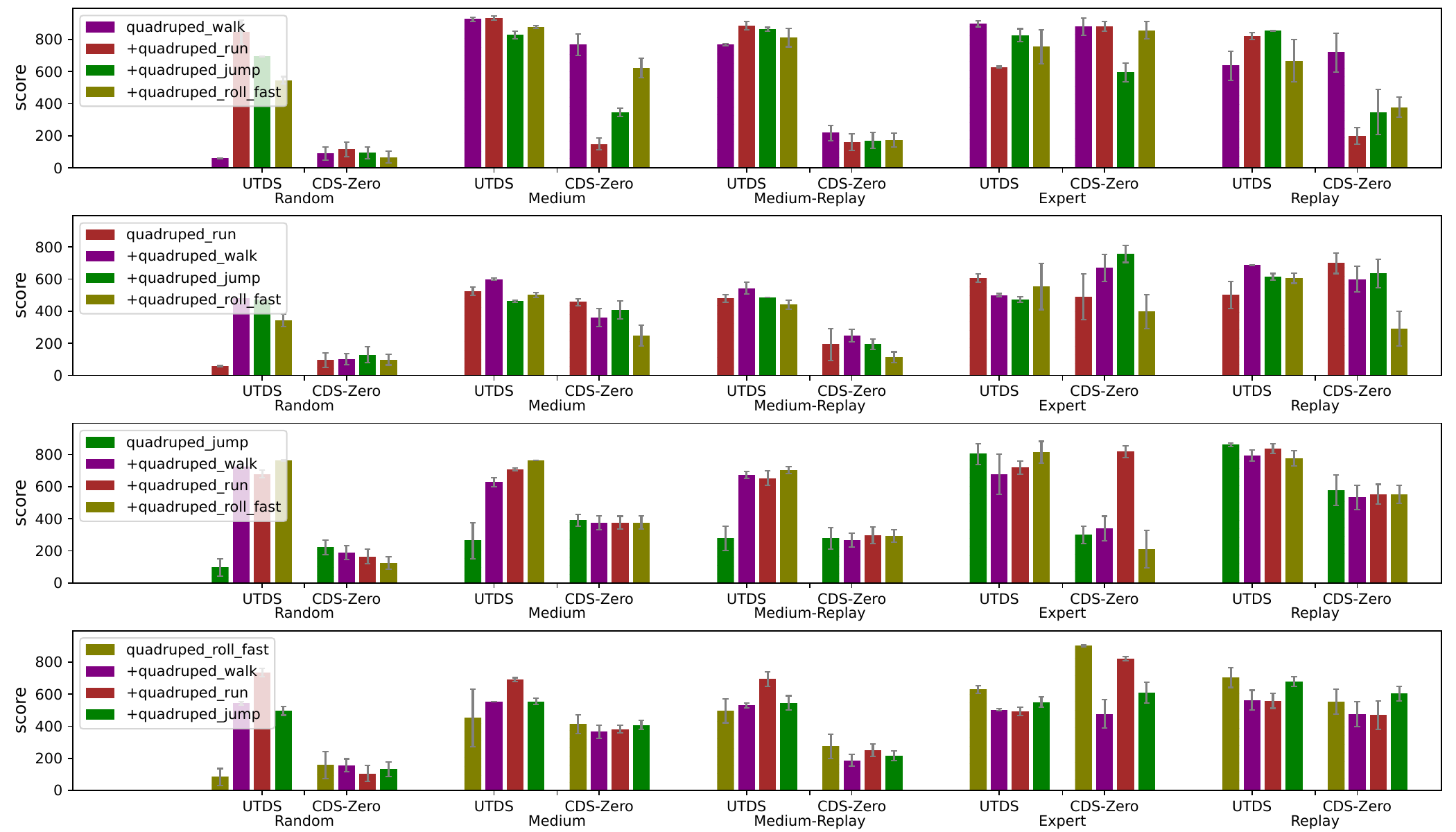}
\caption{Result comparison of the proposed UTDS and CDS-Zero in \emph{Quadruped} domain.}
\label{fig:app-UTDSz-cds-quad}
\end{figure*}

\begin{figure*}[h!]
\centering
\includegraphics[width=0.9\columnwidth]{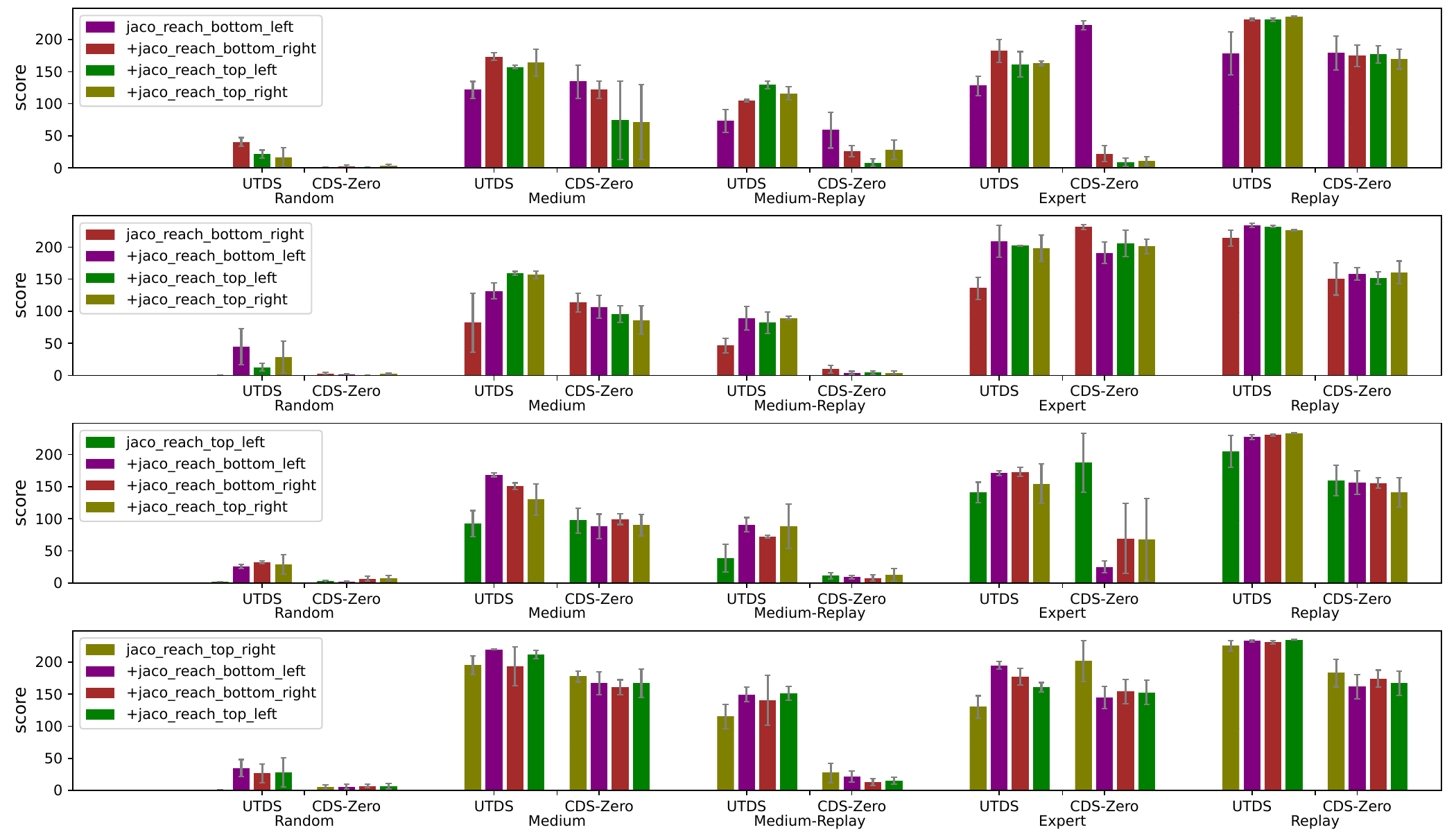}
\caption{Result comparison of the proposed UTDS and CDS-Zero in \emph{Jaco Arm} domain.}
\label{fig:app-UTDSz-cds-jaco}
\end{figure*}

\clearpage

\subsection{Aggregate Evaluation}

We follow the reliable principles \citep{agarwal2021deep} to evaluation the statistical performance. Since the ordinary aggregate measures like \emph{mean} can be easily dominated by a few outlier scores, \citep{agarwal2021deep} presents several robust alternatives that are not unduly affected by outliers and have small uncertainty even with a handful of runs. Based on bootstrap Confidence Intervals (CIs), we can extract aggregate metrics from score distributions, including median, mean, interquartile mean (IQM), and optimality gap. IQM discards the bottom and top 25\% of the runs and calculates the mean score of the remaining 50\% runs. Optimality gap calculates the amount of runs that fail to meet a minimum score of $\eta= 50.0$. We also give performance profiles that reveal performance variability through score distributions. A score distribution shows the fraction of runs above a certain score and is given by $\hat{F}(\tau)=\hat{F}(\tau;x_{1:M,1:N})=\frac{1}{M}\sum_{m=1}^{M}\frac{1}{N}\sum_{n=1}^{N}\mathbbm{1}[x_{m,n}\ge \tau].$ The aggregated results and performance profiles for three domains are given in Fig.~\ref{fig:aggregate-result} and Fig.~\ref{fig:fraction-result}, respectively.

\vspace{5em}
\begin{figure}[h!]
\centering
\subfigure[Walker domain]{
\includegraphics[width=0.9\textwidth]{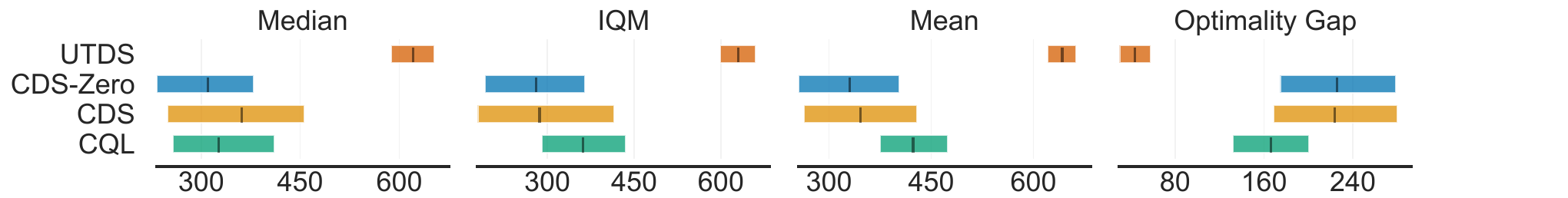}}
\subfigure[Quadruped domain]{
\includegraphics[width=0.9\textwidth]{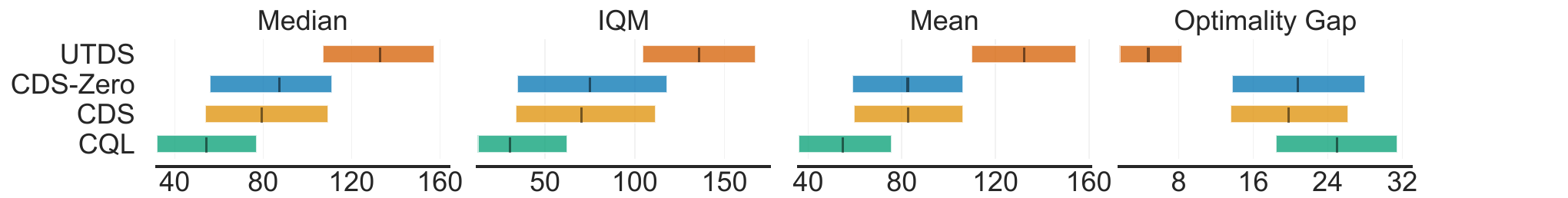}}
\subfigure[Jaco Arm domain]{
\includegraphics[width=0.9\textwidth]{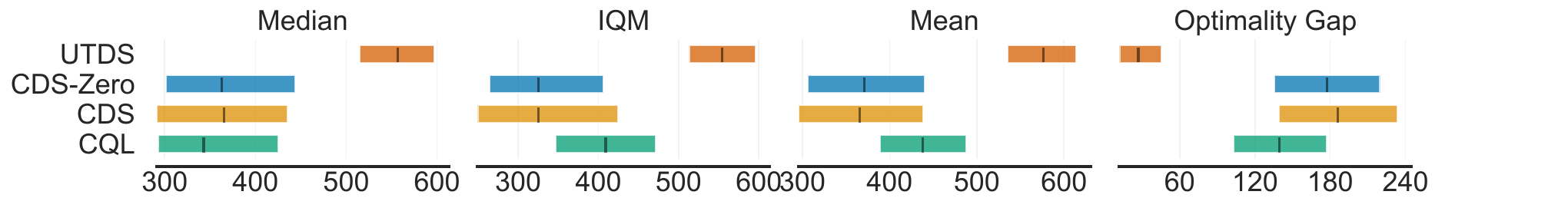} }
\caption{Aggregate metrics with 95\% CIs based on 12 shared tasks among 5 dataset types for each task. Higher mean, median and IQM scores, and lower optimality gap are better. The CIs are estimated using the percentile bootstrap with stratified sampling.}
\label{fig:aggregate-result}
\end{figure}

\begin{figure}[h!]
\centering
\subfigure[Walker domain]{
\includegraphics[width=0.7\textwidth]{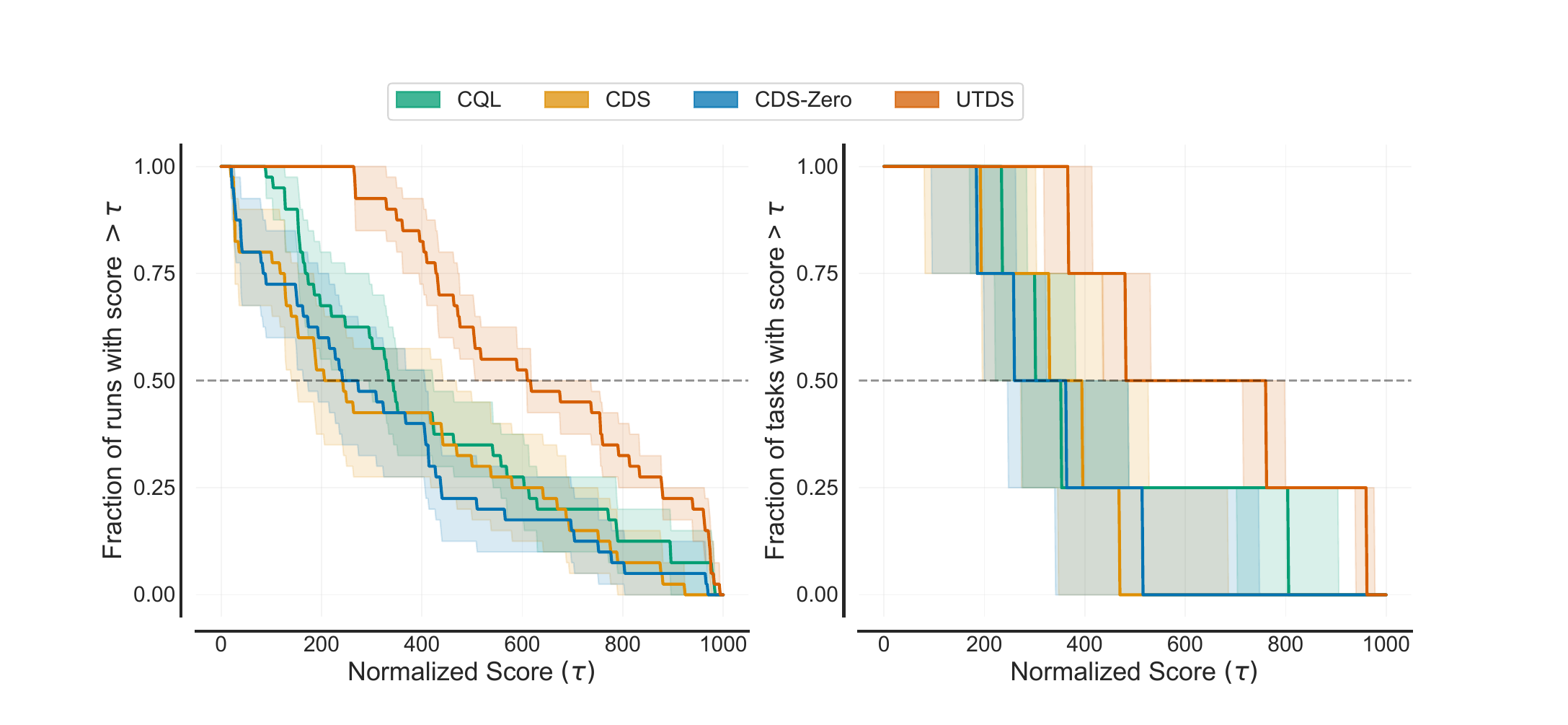}}
\subfigure[Quadruped domain]{
\includegraphics[width=0.7\textwidth]{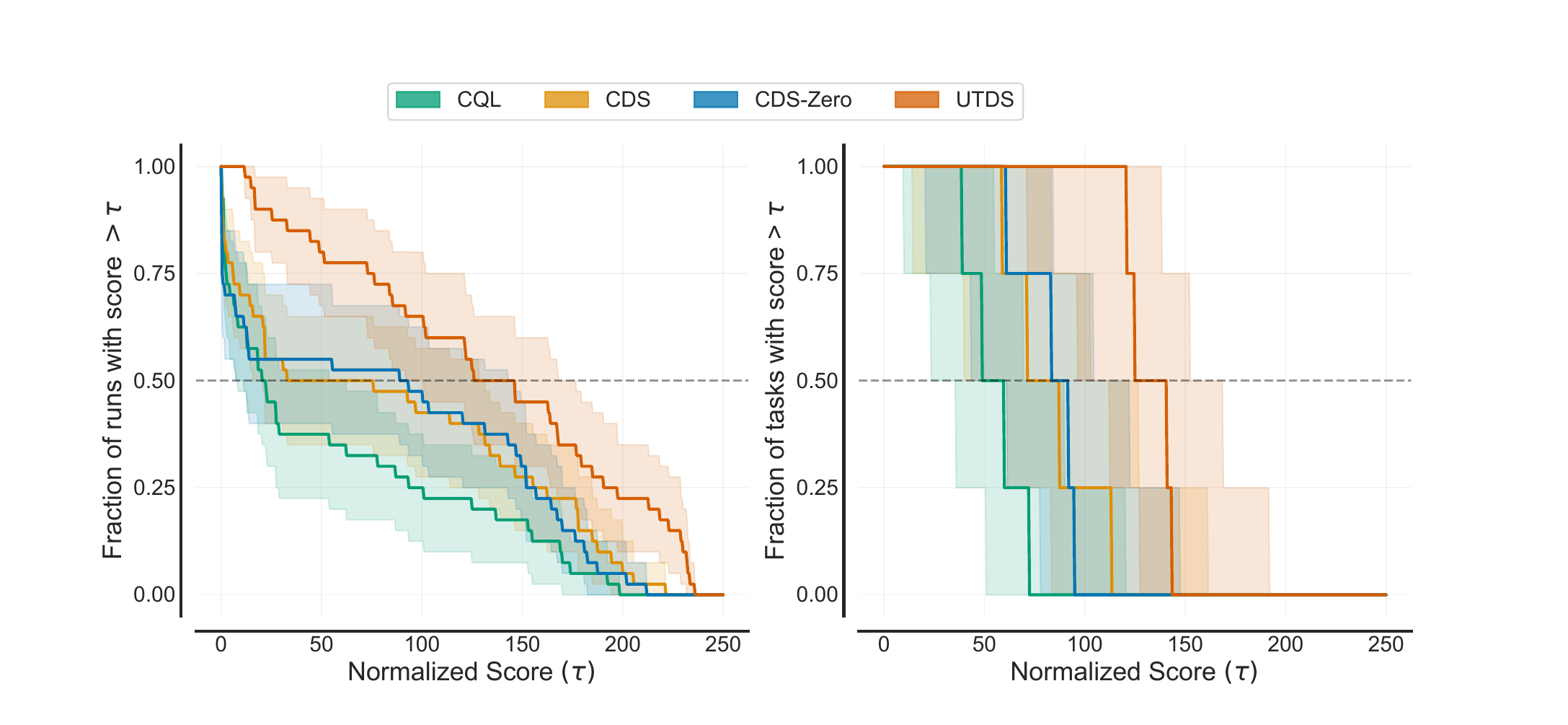}}
\subfigure[Jaco Arm domain]{
\includegraphics[width=0.7\textwidth]{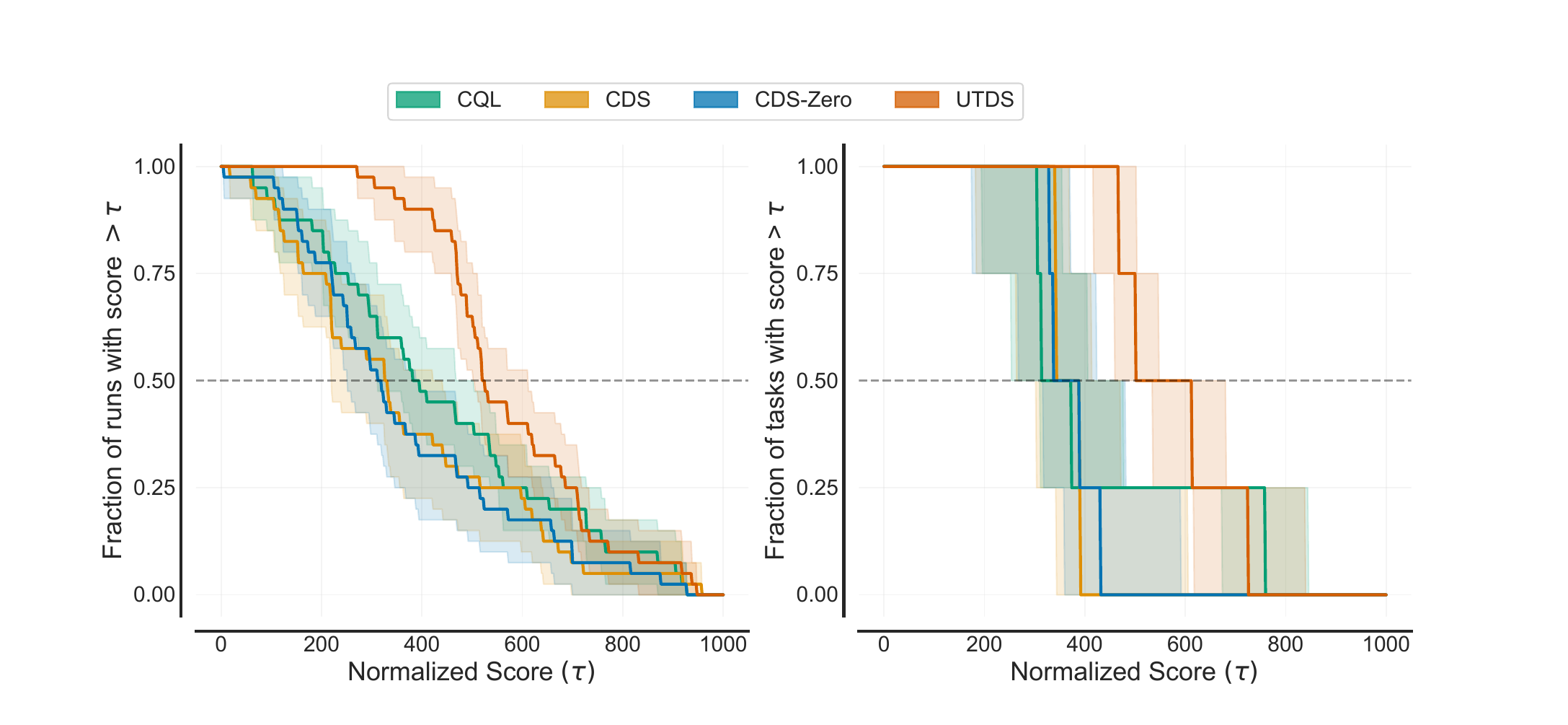} }
\caption{Performance profiles based on score distributions (left), and average score distributions (right). Shaded regions show pointwise 95\% confidence bands based on percentile bootstrap with stratified sampling. The $\tau$ value where the profiles intersect $y = 0.5$ shows the median, and the area under the performance profile corresponds to the mean.}
\label{fig:fraction-result}
\end{figure}

\clearpage
\subsection{\tcb{Ablation Study}}

\tcb{
We remark that $\beta_2$ is used for constructing the pseudo-target for the OOD datapoint, as
\begin{equation}\label{eq:ood-target}
\hT^{\rm ood}Q_i(s,a^{\rm ood}):=Q_i(s,a^{\rm ood})-\beta_2\Gamma_i(s,a^{\rm ood}),
\end{equation}
We adopt an exponentially decay schedule for $\beta_2$. Empirically, we use 3 factors to control the decay process of $\beta_2$, including the initial value, the end value, and the decay factor $\alpha$. In UTDS, we set the initial value to $3.0$, the end value to $0.01$, and the decal factor to $0.99995$. We conducted ablation studies and the result is given in Figure \ref{fig:init}, Figure \ref{fig:end}, and Figure \ref{fig:decay}. 

\begin{figure}[h!]
\centering
\begin{minipage}{0.42\linewidth}
		\centering
		\includegraphics[width=0.9\linewidth]{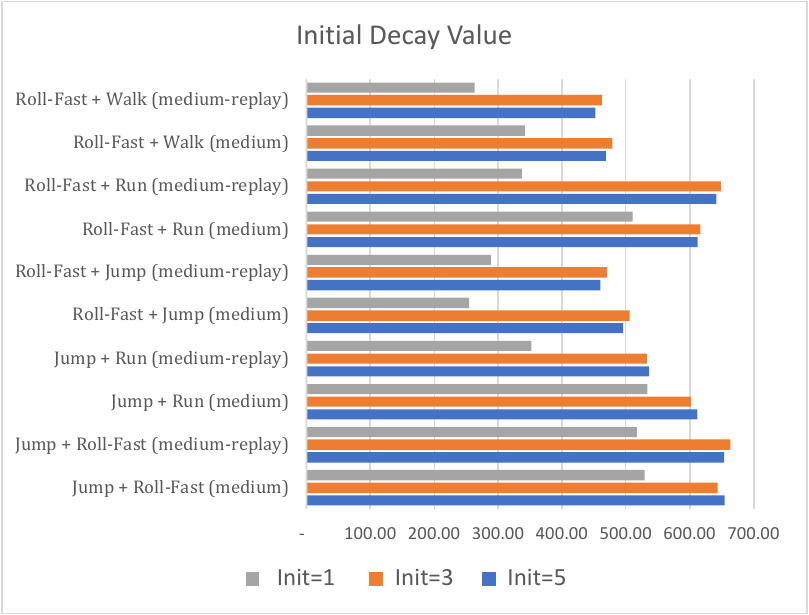}
		\caption{Initial Decay Value}
		\label{fig:init}
\end{minipage}
\begin{minipage}{0.45\linewidth}
		\centering
		\includegraphics[width=0.95\linewidth]{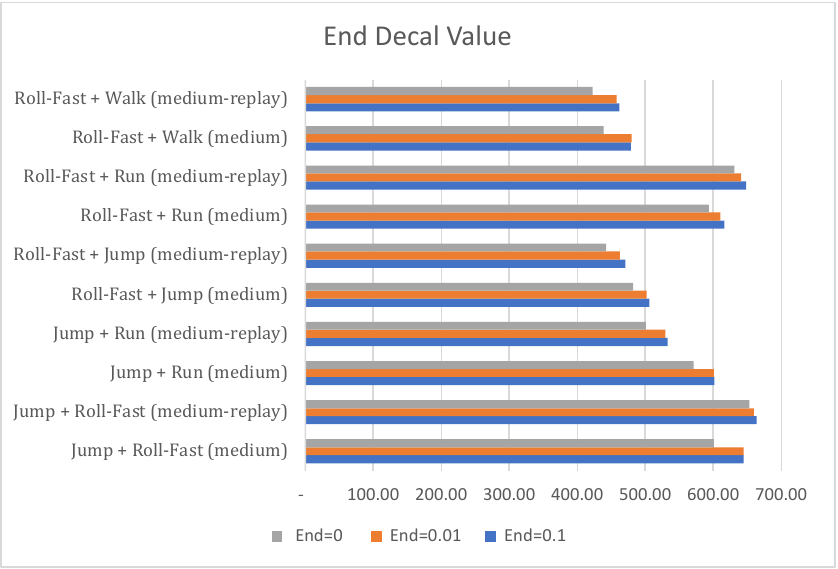}
		\caption{End Decay Value}
		\label{fig:end}
\end{minipage}
\end{figure}

We conduct experiments on several data-sharing tasks from the Quadruped domain. According to Figure 1, we find a large initial value (i.e., 3 or 5) is important for a stable performance. Since the uncertainty quantifier for OOD states and actions is inaccurate initially, a large penalty is required to enforce a strong regularization for the value function of OOD data. In contrast, a small factor will cause over-estimation of OOD data. 

Figure \ref{fig:end} shows the experiment result of different end values in the Quadruped domain. We find that setting the end value larger than $0.01$ is suitable, which provides a small penalty in the final training process. In contrast, decreasing $\beta_2$ to $0$ may cause over-estimation and lead to slightly worse performance.

\begin{figure}[h]
\centering
\includegraphics[width=0.5\textwidth]{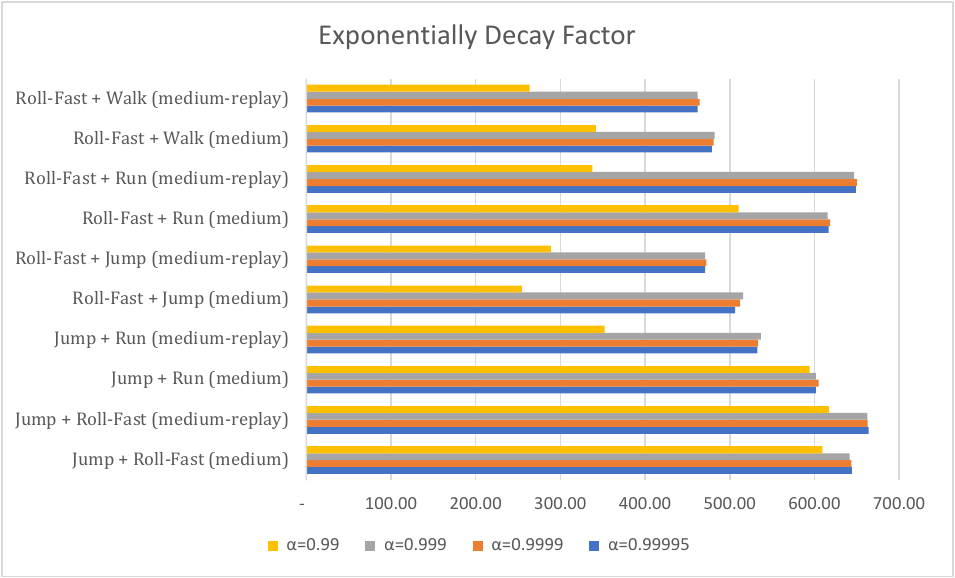}
\caption{Ablation study of the exponentially decay factor of $\beta_2$ in OOD target.}
\label{fig:decay}
\end{figure}

According to our theoretical analysis, an exponential decay factor has well theoretical property to make the OOD operator not lead to negative infinity, which makes a contraction mapping for the OOD target. We try different decay factors in the experiment, the result in Figure~\ref{fig:decay} shows a slow decay process is important to obtain a stable process. As the uncertainty quantifier becomes more accurate in the training process, we decay the penalty factor to make the pseudo-target dependent on the estimated uncertainty gradually.
}

\end{document}